\documentclass[11pt]{article}

\usepackage{ACL2023}
\usepackage{custom}

\newcommand{\mymacro}[1]{{#1}}

\newcommand{\defn}[1]{\textbf{#1}}

\newcommand{\pdens}{{\mymacro{ p}}}

\newcommand{\qdens}{{\mymacro{ q}}}

\newcommand{\precision}{{\mymacro{\psi}}}
\newcommand{\precisionFun}[1]{{\precision\left(#1\right)}}

\newcommand{\ind}[1]{\mathbbm{1} \left\{ #1 \right\}}

\newcommand{\Q}{{\mymacro{ \mathbb{Q}}}}
\newcommand{\R}{{\mymacro{ \mathbb{R}}}}
\newcommand{\Rex}{{\mymacro{ \overline{\R}}}}

\newcommand{\Qnonnegative}{{\mymacro{ \mathbb{Q}_{\geq 0}}}}

\newcommand{\vfunc}{{\mymacro{ \boldsymbol{f}}}}

\newcommand{\abs}[1]{{\mymacro{ \left| #1 \right|}}}
\newcommand{\norm}[1]{{\mymacro{ \left\lVert #1 \right\rVert}}}

\newcommand{\ordering}{{\mymacro{ n}}}

\newcommand{\alphabet}{{\mymacro{ \Sigma}}}

\newcommand{\eosalphabet}{{\mymacro{ \overline{\alphabet}}}}

\newcommand{\kleene}[1]{{\mymacro{#1^*}}}

\newcommand{\str}{{\mymacro{\boldsymbol{y}}}}

\newcommand{\strlt}{{\mymacro{ \str_{<\tstep}}}}
\newcommand{\strlet}{{\mymacro{ \str_{\leq\tstep}}}}

\newcommand{\strlen}{{\mymacro{T}}}

\newcommand{\sym}{{\mymacro{y}}}
\newcommand{\eossym}{{\mymacro{\overline{\sym}}}}
\newcommand{\syma}{{\mymacro{\texttt{a}}}}
\newcommand{\symb}{{\mymacro{\texttt{b}}}}

\newcommand{\symz}{{\mymacro{\texttt{z}}}}

\newcommand{\defeq}{\mathrel{\stackrel{\textnormal{\tiny def}}{=}}}

\newcommand{\NTo}[1]{{\mymacro{\left[ #1 \right]}}}

\newcommand{\set}[1]{{\mymacro{\left\{ #1 \right\}}}}

\newcommand{\idx}{{\mymacro{ n}}}
\newcommand{\idxn}{{\mymacro{ n}}}
\newcommand{\idxd}{{\mymacro{ d}}}
\newcommand{\idxi}{{\mymacro{ i}}}
\newcommand{\idxj}{{\mymacro{ j}}}

\newcommand{\setsize}{{\mymacro{ N}}}

\newcommand{\nstates}{{\mymacro{ |\states|}}}
\newcommand{\nsymbols}{{\mymacro{ |\alphabet|}}}
\newcommand{\eosnsymbols}{{\mymacro{ |\eosalphabet|}}}
\newcommand{\tstep}{{\mymacro{ t}}}

\newcommand{\pLM}{\mymacro{\pdens}}
\newcommand{\pLMA}{\mymacro{\pdens}_\automaton}
\newcommand{\pLMAFun}[2]{\mymacro{{\pdens}_\automaton\left(#1\mid#2\right)}}
\newcommand{\pLMR}{\mymacro{\pdens}_\rnn}
\newcommand{\pLMRFun}[2]{\mymacro{{\pdens}_\rnn\left(#1\mid#2\right)}}
\newcommand{\qLM}{\mymacro{\qdens}}
\newcommand{\pLNSM}{\mymacro{\pdens}}

\newcommand{\pLN}{\mymacro{\pdens}}

\newcommand{\eos}{{\mymacro{\textsc{eos}}}}

\newcommand{\embedDim}{{\mymacro{ R}}}

\newcommand{\onehot}[1]{{\mymacro{ \llbracket#1\rrbracket}}}

\newcommand{\inEmbedding}{{\mymacro{ \vr}}}
\newcommand{\inEmbeddingFun}[2][]{{\mymacro{ \inEmbedding\!\left(#2\right)}}}

\newcommand{\inEmbedSym}{{\mymacro{ \inEmbeddingFun{\sym}}}}
\newcommand{\inEmbedSymt}{{\mymacro{ \inEmbeddingFun{\sym_\tstep}}}}

\newcommand{\symt}{{\mymacro{ \sym_{\tstep}}}}

\newcommand{\symtminus}{{\mymacro{ \sym_{\tstep-1}}}}

\newcommand{\symoverline}{{\mymacro{\overline{\sym}}}}

\newcommand{\bias}{{\mymacro{ \vb}}}

\newcommand{\biasVech}{{\mymacro{ \vb}}}

\newcommand{\one}{{\mymacro{\mathbf{1}}}}

\newcommand{\automaton}{{\mymacro{ \mathcal{A}}}}
\newcommand{\wfsa}{{\mymacro{ \automaton}}}

\newcommand{\stateq}{{\mymacro{ q}}}

\newcommand{\states}{{\mymacro{ Q}}}

\newcommand{\trans}{{\mymacro{ \tau}}}

\newcommand{\prevq}{{\mymacro{ \varphi }}}

\newcommand{\weight}{{\mymacro{ \textnormal{w}}}}
\newcommand{\weightv}{{\mymacro{ w }}}

\newcommand{\prefixweight}{{\mymacro{\widetilde\weight}}}
\newcommand{\apath}{{\mymacro{ \boldsymbol p}}}

\newcommand{\pathlen}{{\mymacro{ N}}}
\newcommand{\paths}{{\mymacro{ P}}}

\newcommand{\initf}{{\mymacro{ \lambda}}}
\newcommand{\finalf}{{\mymacro{ \rho}}}
\newcommand{\initfFun}[1]{{\mymacro{\initf\left(#1\right)}}}
\newcommand{\finalfFun}[1]{{\mymacro{\finalf\left(#1\right)}}}

\newcommand{\wfsatuple}{{\mymacro{ \left( \alphabet, \states, \trans, \initf, \finalf \right)}}}

\newcommand{\edge}[4]{{\mymacro{#1 \xrightarrow{#2 / #3} #4}}}

\newcommand{\yield}{{\mymacro{\textbf{s}}}}

\newcommand{\elmanrnntuple}{{\mymacro{ \left(\Q^\hiddDim, \alphabet, \sigmoid, \recMtx, \inMtx, \biasVech, \initstate\right)}}}

\newcommand{\rnn}{{\mymacro{ \mathcal{R}}}}

\newcommand{\ernnAcr}{{\mymacro{ERNN}}\xspace}
\newcommand{\aernnAcr}{{\mymacro{non-linear output Elman LM}}\xspace}
\newcommand{\dornnAcr}{{\mymacro{deep output Elman LM}}\xspace}

\newcommand{\pfsaAcr}{{\mymacro{PFSA}}\xspace}
\newcommand{\deltapfsaAcr}{{\mymacro{$\delta$-perturbed \pfsaAcr}}\xspace}

\newcommand{\fslmAcr}{{\mymacro{regular LM}}\xspace}
\newcommand{\recMtx}{{\mymacro{ \mU}}}

\newcommand{\inMtx}{{\mymacro{ \mV}}}
\newcommand{\outMtx}{{\mymacro{ \mE}}}

\newcommand{\hiddDim}{{\mymacro{ D}}}

\newcommand{\hiddState}{{\mymacro{ \vh}}}
\newcommand{\hiddStatet}{{\mymacro{ \hiddState_\tstep}}}

\newcommand{\hiddStatetminus}{{\mymacro{ \hiddState_{\tstep - 1}}}}
\newcommand{\hiddStatetzero}{{\mymacro{ \hiddState_{0}}}}

\newcommand{\initstate}{{\mymacro{\boldsymbol{\eta}}}}

\newcommand{\vhzero}{{\mymacro{ \vh_0}}}

\newcommand{\hiddStateZero}{{\mymacro{ \vhzero}}}

\newcommand{\vht}{{\mymacro{ \vh_t}}}

\newcommand{\vhtminus}{{\mymacro{ \vh_{t-1}}}}

\newcommand{\statedistributionFun}[2]{\pLMAFun{#1}{#2}}

\newcommand{\mlp}{{\mymacro{\mathbf{F}}}}
\newcommand{\mlpFun}[1]{\mlp\left(#1\right)}

\newcommand{\Simplexnminus}{{\mymacro{ \boldsymbol{\Delta}^{N-1}}}}

\newcommand{\SimplexEosalphabetminus}{{\mymacro{ \boldsymbol{\Delta}^{|\eosalphabet|-1}}}}

\DeclareMathSymbol{\mlq}{\mathord}{operators}{``} 
\DeclareMathSymbol{\mrq}{\mathord}{operators}{`'} 

\newcommand{\negterm}[1]{{\mymacro{ {\raise.17ex\hbox{$\scriptstyle\sim$}} #1}}}

\newcommand{\ifcondition}{\textbf{if }}
\newcommand{\ifcond}{\ifcondition}
\newcommand{\otherwisecondition}{\textbf{otherwise }}

\newcommand{\ignore}[1]{}
\newcommand{\expandLater}[1]{}

\def\1{\mathbf{1}}

\def\eps{{\mymacro{ \varepsilon}}}

\def\rvz{{{\mymacro{ \mathbf{z}}}}}

\def\vb{{{\mymacro{ \mathbf{b}}}}}

\def\vd{{{\mymacro{ \mathbf{d}}}}}

\def\vh{{{\mymacro{ \mathbf{h}}}}}

\def\vr{{{\mymacro{ \mathbf{r}}}}}

\def\vv{{{\mymacro{ \mathbf{v}}}}}

\def\vx{{{\mymacro{ \mathbf{x}}}}}
\def\vy{{{\mymacro{ \mathbf{y}}}}}

\def\evx{{{\mymacro{ x}}}}

\def\mE{{{\mymacro{ \mathbf{E}}}}}

\def\mT{{{\mymacro{ \mathbf{T}}}}}
\def\mU{{{\mymacro{ \mathbf{U}}}}}
\def\mV{{{\mymacro{ \mathbf{V}}}}}
\def\mW{{{\mymacro{ \mathbf{W}}}}}
\def\mX{{{\mymacro{ \mathbf{X}}}}}
\def\mY{{{\mymacro{ \mathbf{Y}}}}}

\def\sK{{{\mymacro{ \mathcal{K}}}}}

\newcommand{\N}{{\mymacro{ \mathbb{N}}}}
\newcommand{\Nzero}{{\mymacro{ \mathbb{N}_{\geq 0}}}}

\newcommand{\projfunc}{{\mymacro{\boldsymbol{\pi}}}}
\newcommand{\projfuncEosalphabetminus}{{\mymacro{\projfunc}}}
\newcommand{\projfuncEosalphabetminusFunc}[1]{{\mymacro{\projfunc\left(#1\right)}}}

\newcommand{\softmax}{{\mymacro{ \mathrm{softmax}}}}
\newcommand{\softmaxshort}{{\mymacro{\boldsymbol{\sigma}}}}
\newcommand{\sparsemax}{{\mymacro{ \mathrm{sparsemax}}}}
\newcommand{\ReLU}{{\mymacro{ \mathrm{ReLU}}}}
\newcommand{\softmaxfunc}[2]{{\mymacro{ \softmaxshort\!\left(#1\right)_{#2}}}} 
\newcommand{\sparsemaxfunc}[2]{{\mymacro{ \mathrm{sparsemax}\!\left(#1\right)_{#2}}}} 

\newcommand{\heaviside}{{\mymacro{ H}}}

\newcommand{\sigmoid}{{\mymacro{\sigma}}}

\DeclareMathOperator*{\argmax}{{\mymacro{ argmax}}}
\DeclareMathOperator*{\argmin}{{\mymacro{ argmin}}}

\newcommand{\bigO}[1]{{\mymacro{ \mathcal{O}\left(#1\right)}}}

\newcommand{\entropy}[1]{{\mymacro{\mathrm{H}(#1)}}}
\newcommand{\family}{\mymacro{\mathcal{F}}}

\newcommand{\tvdacr}{\mymacro{\textsc{tvd}}}
\newcommand{\tvd}[2]{\mymacro{\tvdacr(#1,#2)}}
\newcommand{\rtvd}[3]{\mymacro{\tvdacr_#1(#2,#3)}}

\newcommand{\automatondelta}{\mymacro{\automaton_\delta}}
\newcommand{\automatoneta}{\mymacro{\automaton_\eta}}
\newcommand{\pLMdelta}{\mymacro{\pLM_{\automatondelta}}}
\newcommand{\pLMdeltaN}[1]{\mymacro{\pLM_{\automaton_{\delta_{#1}}}}}
\newcommand{\pLMeta}{\mymacro{\pLM_{\automatoneta}}}

\newcommand{\invtemp}{\mymacro{\lambda}}

\newcommand{\activation}{\mymacro{\boldsymbol{\alpha}}}
\newcommand{\mlpActivation}{\mymacro{\boldsymbol{\beta}}}
\newcommand{\fstr}{\mymacro{f_\str}}

\title{Lower Bounds on the Expressivity of Recurrent Neural Language Models}

\author{
Anej Svete\thanks{Equal contribution.}%
~\;~\;~
Franz Nowak\footnotemark[1]%
~\;~\;~
Anisha Mohamed Sahabdeen%
~\;~\;~Ryan Cotterell\\
\texttt{\{\href{mailto:asvete@ethz.ch}{asvete}, \href{mailto:fnowak@ethz.ch}{fnowak},
\href{mailto:amohame@ethz.ch}{amohame}, \href{mailto:ryan.cotterell@ethz.ch}{ryan.cotterell}\}@ethz.ch}\\
    {%
\setlength{\fboxsep}{2.5pt}%
\setlength{\fboxrule}{2.5pt}%
\fcolorbox{white}{white}{
    \includegraphics[width=.15\linewidth]{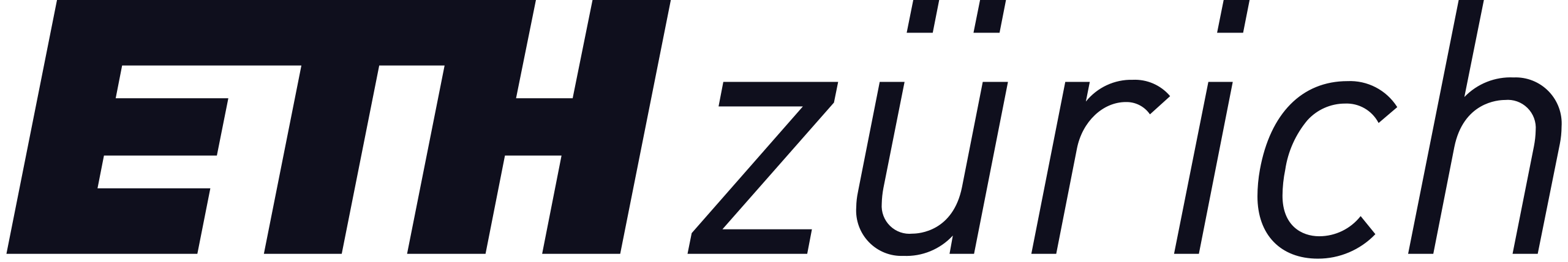}
}
}}

\begin{document}
\renewcommand*{\thefootnote}{\fnsymbol{footnote}}
\maketitle
\renewcommand*{\thefootnote}{\arabic{footnote}}

\begin{abstract}
    The recent successes and spread of large neural language models (LMs) call for a thorough understanding of their computational ability.
    Describing their computational abilities through LMs' \emph{representational capacity} is a lively area of research.
    However, investigation into the representational capacity of neural LMs has predominantly focused on their ability to \emph{recognize} formal languages.
    For example, recurrent neural networks (RNNs) with Heaviside activations are tightly linked to regular languages, i.e., languages defined by finite-state automata (FSAs).
    Such results, however, fall short of describing the capabilities of RNN \emph{language models} (LMs), which are definitionally \emph{distributions} over strings.
    We take a fresh look at the representational capacity of RNN LMs by connecting them to \emph{probabilistic} FSAs and demonstrate that RNN LMs with linearly bounded precision
    can express arbitrary regular LMs.\looseness=-1
    
    \vspace{0.5em}
    {\includegraphics[width=1.36em,height=1.25em]{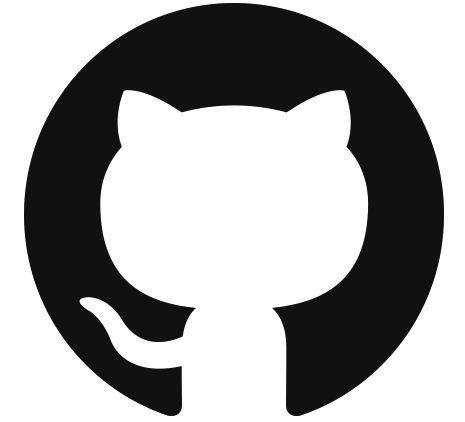}\hspace{1em}\parbox{\dimexpr\linewidth-2\fboxsep-2\fboxrule}{\url{https://github.com/rycolab/nondeterministic-rnns}}}
\end{abstract}

\section{Introduction} \label{sec:intro}
Neural language models (LMs) excel at many NLP tasks, e.g., machine translation, conversational AI, text classification, and natural language inference \citep{cho-etal-2014-learning, albalak-etal-2022-feta, sun-etal-2023-text, schick-schutze-2021-exploiting}.
Their strong empirical performance suggests the need for theory to explain what LMs can and cannot do in a formal sense.
Such theory should serve to better understand and alleviate the practical limitations of neural LMs.
However, the most common paradigm to perform such an analysis of neural LMs revolves around proving which formal languages common LM architectures can express \citep[\textit{inter alia}]{ackerman2020survey, icard-2020-calibrating, merrill-etal-2020-formal, perez-etal-2021-attention}, i.e., to characterize the representational capacity of neural LMs with tools from the theory of computation.\looseness=-1

\begin{figure*}[t]
    \centering
    \footnotesize
    \begin{tikzpicture}[auto,
            start chain = going right,
            box/.style = {draw, rounded corners, blur shadow,fill=white, on chain, align=center, minimum height=1cm, minimum width=1.5cm}]
        \node[box,fill=ETHBlue!20] (b1)    {Exact simulation \\ with $\sparsemax$ and \\ $\norm{\cdot}_1$ normalization.};
        \node[draw=none] (t1) [above = 3mm of b1]   {\textcolor{ETHBlue}{\cref{thm:sparsemax}}};
        \node[box,fill=ETHBlue!20] (b2)   [right = 1cm of b1] {Exact simulation \\ with softmax, \\ $\log$ activation, and $-\infty$.};
        \node[draw=none] (t2) [above = 3mm of b2]   {\textcolor{ETHBlue}{\cref{thm:softmax}}};
        \node[box,fill=ETHBlue!20] (b3)   [right = 1cm of b2] {Approximate simulation \\ with softmax and \\ $\log$ activation.};
        \node[draw=none] (t3) [above = 3mm of b3]   {\textcolor{ETHBlue}{\cref{thm:aernn-approx-simulation}}};
        \node[box,fill=ETHBlue!20] (b4)   [right = 1cm of b3] {Approximate simulation \\ with softmax \\ and an MLP.};
        \node[draw=none] (t4) [above = 3mm of b4]   {\textcolor{ETHBlue}{\cref{thm:dornn-approximation-proof}}};
        \begin{scope}[dashed, rounded corners,-{Latex[length=1.8mm, width=1.4mm]}]
            \path       (b1) edge (b2);
            \path       (b2) edge (b3);
            \path       (b3) edge (b4);
        \end{scope}
    \end{tikzpicture}
    \caption{A roadmap of the results in this paper describing the representational capacity of Elman LMs.
    We start with an intuitive result showing how a sparsemax RNN with an $\norm{\cdot}_1$-normalized hidden state can exactly simulate a \fslmAcr{}.
    We then bring these modeling choices closer to practical implementations, leading us to the final result, which shows how sparsemax Elman LMs with an MLP-transformed hidden state can approximate arbitrary \fslmAcr{}s.\looseness=-1}
    \label{fig:roadmap}
    \vspace{-15pt}
\end{figure*}
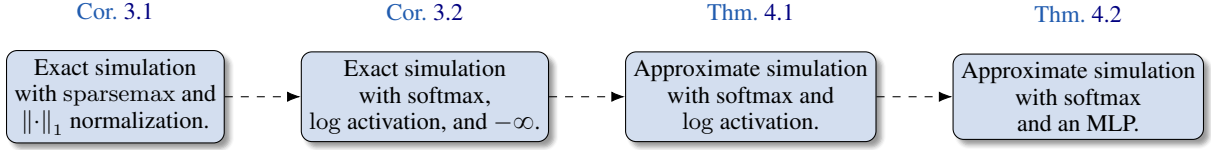

Formally, a language model is a probability distribution over $\kleene{\alphabet}$, the set of all strings from some alphabet $\alphabet$.\footnote{An alphabet is a finite, non-empty set of symbols.}
We say that two language models are weakly equivalent if they express the same distribution over strings.
\pfsaAcr{}s (both deterministic and non-deterministic) and recurrent neural LMs, e.g., those derived from Elman recurrent neural networks \citep[RNNs;][]{Elman1990}, however, are instances of \emph{families} of language models, i.e., sets of distributions over strings that can be represented under the specific formalism.
For instance, if one changes the parameters of a \pfsaAcr{}, one arrives at a different language model.
Therefore, going beyond the formal comparison of individual LMs, we may also desire to make formal statements about families of models.
One natural way to construct an (upper) bound on the representational capacity of an entire family of language models $\family_1$, e.g., those expressible by RNNs, is to find another family of language models $\family_2$ and show that, for every language model $q \in \family_2$, there exists a weakly equivalent $p \in \family_1$.
Our goal in this paper is to show that, for every non-deterministic \pfsaAcr{}, there exists a weakly equivalent RNN LM.\looseness=-1

The approach to studying language models presented in the previous paragraph differs fundamentally from most work that uses formal language theory to study the representational capacity of language models. 
Indeed, most papers try to relate language models to language \emph{acceptors} \citep[\emph{inter alia}]{Siegelmann1992OnTC, merrill-etal-2020-formal, perez-etal-2021-attention, merrill-etal-2022-saturated}.
E.g., finite-state automata have been linked to RNNs extensively before \citep{noga-1991-efficient, indyk-1995-optimal, merrill-2019-sequential, merrill2022extracting}.
However, acceptors are fundamentally different objects than LMs, as they define \emph{sets} of strings rather than distributions over them.
Thus, analyzing acceptors, does not address our desire to understand recurrent neural LMs as probability distributions over $\kleene{\alphabet}$.

In previous work, \citet{svete2023recurrent} demonstrate a particularly close relationship between recurrent neural language models and \emph{deterministic} \pfsaAcr{}s.
However, their theory does not extend to the family of non-deterministic \pfsaAcr{}s because it is a basic fact of probabilistic formal language theory that not all \pfsaAcr{}s can be determinized \citep{Allauzen2003EfficientAF}.
To date, no corresponding result exists for the non-deterministic case.
To fill this void, we give a construction that shows real-time $\ReLU$-activated Elman RNNs (Elman LMs) with linearly bounded precision can exactly simulate or arbitrary approximate \emph{non}-deterministic \pfsaAcr{}s depending on the components employed in the network.
Concretely, we show that \begin{enumerate*}[label=\textit{(\arabic*)}]
    \item Elman LMs with the sparsemax \citep{sparsemax} function can represent any regular LM exactly.
    \item Elman LMs with the softmax projection function and a nonlinear output function can approximate a regular LM arbitrarily well.
\end{enumerate*}
The roadmap of the paper is presented in \cref{fig:roadmap}.\looseness=-1

\section{Preliminaries} \label{sec:preliminaries}
We begin by introducing some preliminaries.

\subsection{Language Modeling}
Most modern neural LMs define the probability $\pLM\left(\str\right)$ of a string $\str \in \kleene{\alphabet}$ as a product of conditional probability distributions $\pLNSM$, i.e.,
\begin{equation} \label{eq:lnlm}
    \pLN\left(\str\right) \defeq \pLNSM\left(\eos\mid\str\right) \prod_{\tstep = 1}^{|\str|} \pLNSM\left(\symt \mid \strlt\right),
\end{equation}
where $\eos \notin \alphabet$ is a special \underline{e}nd-\underline{o}f-\underline{s}equence symbol.
A language model expressed as in \Cref{eq:lnlm} is called \defn{autoregressive}.
Let $\eosalphabet \defeq \alphabet \cup \left\{\eos\right\}$.
Then, each $\pLNSM\left(\eossym_\tstep \mid \strlt\right)$ is a distribution over $\eosalphabet$.
Additionally, we may also consider $\varepsilon$-augmented language models where each $\pLNSM$ is a distribution over $\eosalphabet \cup \{ \varepsilon\}$ where $\varepsilon \not\in \alphabet$ is a special symbol not in the alphabet that represents the empty string.
This allows the language model to perform computations \emph{longer} than the length of the string it generates.
An autoregressive language model is called \defn{real-time} if each $\pLNSM$ is \emph{only} distribution over $\eosalphabet$, i.e., not over $\eosalphabet \cup \{ \varepsilon\}$.
Throughout this paper, we construct and investigate the expressivity of real-time autoregressive language models based on RNNs.

At a high level, we are interested in encoding real-time LMs as RNN LMs.
To do so, we need a notion of equivalence between language models.
In this paper, we will work with weak equivalence.
\begin{definition} \label{def:weak-equivalence}
    Two LMs $\pLM$ and $\qLM$ over $\kleene{\alphabet}$ are \defn{weakly equivalent} if $\pLM\left(\str\right) = \qLM\left(\str\right)$ for all $\str \in \kleene{\alphabet}$.\looseness=-1
\end{definition}

\subsection{Regular Language Models} \label{sec:wfsas}
Probabilistic finite-state automata are a well-understood real-time computational model.\looseness=-1
\begin{definition}\label{def:stochastic-wfsa}
  A \defn{probabilistic finite-state automaton} (\pfsaAcr{}) is a 5-tuple $\wfsatuple$ where $\alphabet$ is an alphabet, $\states$ is a finite set of states, $\trans \subseteq \states \times \alphabet \times \Qnonnegative \times \states$ is a finite set of weighted transitions
    where we write transitions $\left(\stateq, \sym, w, \stateq^\prime\right) \in \trans$ as $\edge{\stateq}{\sym}{w}{\stateq^\prime}$,\footnote{We further assume a $(\stateq, \sym, \stateq^\prime)$ triple appears in at most \emph{one} element of $\trans$.\looseness=-1}
    and $\initf, \finalf\colon \states \rightarrow \Qnonnegative$ are functions that assign each state its initial and final weight, respectively.
    Moreover, for all states $\stateq \in \states$, $\trans, \initf$ and $\finalf$ satisfy $\sum_{\stateq \in \states} \initf\left(\stateq\right) = 1$, and $\sum\limits_{\edge{\stateq}{\sym}{w}{\stateq^\prime} \in \trans} w + \finalf\left(\stateq\right) = 1$.
\end{definition}

We next define some basic concepts. 
A \pfsaAcr{} $\automaton = \wfsatuple$ is \defn{deterministic} if $|\set{\stateq \mid \initfFun{\stateq} > 0}| = 1$ and, for every $\stateq \in \states, \sym \in \alphabet$, there is at most one $\stateq^\prime \in \states$ such that $\edge{\stateq}{\sym}{w}{\stateq^\prime} \in \trans$ with $w > 0$.%
\footnote{In this paper, we do \emph{not} distinguish between a transition for a given symbol with weight $w=0$ and the absence of a transition on that symbol. 
That is, we assume there always exists a transition $\edge{\stateq}{\sym}{w}{\stateq'} \in \trans$ for any $\stateq, \stateq' \in \states$ and $\sym \in \alphabet$, albeit possibly with $w = 0$.
Such a choice turns out to be useful in our technical exposition; see \Cref{alg:deltapfsa}.
}
Any state $\stateq$ where $\initfFun{\stateq}>0$ is called an \defn{initial state}, and if $\finalfFun{\stateq} > 0$, it is called a \defn{final state}.
A \defn{path} $\apath$ of length $\pathlen$ is a sequence of subsequent transitions in $\automaton$, denoted as\looseness=-1
\begin{equation}
    \!\edge{\stateq_1}{\sym_1}{w_1}{\edge{\stateq_2}{\sym_2}{w_2}{\stateq_3} \!\cdots\! \edge{\stateq_{\pathlen}}{\sym_{\pathlen}}{w_{\pathlen}}{\stateq_{\pathlen + 1}}}.
\end{equation}
The \defn{yield} of a path is $\yield\left(\apath\right)\defeq \sym_1 \ldots \sym_{\pathlen}$.
The \defn{prefix weight} $\prefixweight$ of a path $\apath$ is the product of the transition and initial weights, whereas the \defn{weight} of a path additionally has the final weight multiplied in.
In symbols, this means

\noindent\begin{minipage}{0.49\linewidth}
\begin{equation} \label{eq:prefix-path-weight}
    \prefixweight(\apath)\defeq \prod_{\idx = 0}^\pathlen w_\idx,
\end{equation}
\end{minipage}
\begin{minipage}{0.49\linewidth}
\begin{equation}
    \weight(\apath)\defeq \prod_{\idx = 0}^{\pathlen+1} w_\idx,
\end{equation}
\end{minipage}
with $w_0 \defeq \initf(\stateq_1)$ and $w_{\pathlen+1} \defeq \finalf(\stateq_{\pathlen+1})$.
We write $\paths(\automaton)$ for the set of all paths in $\automaton$ and we write $\paths(\automaton, \str)$ for the set of all paths in $\automaton$ with yield $\str$.
The sum of weights of all paths that yield a certain string $\str\in\kleene{\alphabet}$ is called the \defn{stringsum}, given in the notation below
\begin{equation}
\automaton \left( \str \right) \defeq \sum_{\apath \in \paths\left( \automaton, \str \right) }  \weight \left( \apath \right).
\end{equation}
The stringsum gives the probability of the string $\str$.
We say a state $\stateq \in \states$ is \defn{accessible} if there exists a path with non-zero weight from an initial state to $\stateq$. 
A state $\stateq \in \states$ is \defn{co-accessible} if there exists a path with non-zero weight from $\stateq$ to a final state.
An automaton in which all states are accessible and co-accessible is called \defn{trim}.

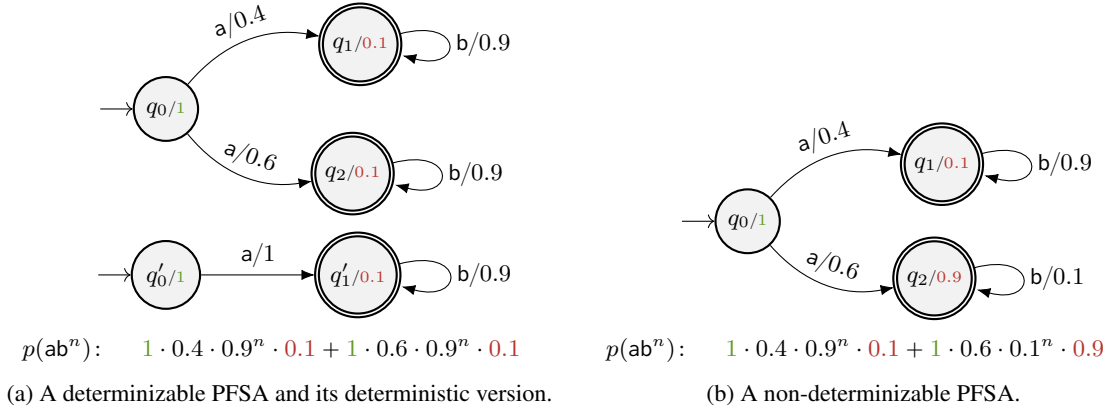
\begin{figure*}[t!]
    \centering
    \begin{subfigure}{0.45\textwidth}
    \begin{tikzpicture}[node distance=8mm, minimum size=12mm]
        \footnotesize
        \node[state, initial] (q0) [] { $\stateq_0 {\scriptstyle/\textcolor{ETHGreen}{1}}$ };
        \node[state, accepting] (q1) [above right = of q0, xshift=13mm, yshift=-4mm] { $\stateq_1{\scriptstyle/\textcolor{ETHRed}{0.1}}$ };
        \node[state, accepting] (q2) [below right = of q0, xshift=12mm, yshift=4mm] { $\stateq_2{\scriptstyle/\textcolor{ETHRed}{0.1}}$ };
        \draw[transition] (q0) edge[auto, bend left, sloped] node[minimum size=4mm]{ $\syma/{0.4}$ } (q1)
        (q0) edge[auto, bend right, sloped] node[minimum size=4mm]{ $\syma/{0.6}$ } (q2)
        (q1) edge[auto, loop right] node[minimum size=4mm]{ $\symb/{0.9}$ } (q1)
        (q2) edge[auto, loop right] node[minimum size=4mm]{ $\symb/{0.9}$ } (q2) ;
        
        \node[state, initial] (q10) [below = 13 mm of q0] { $\stateq'_0{\scriptstyle/\textcolor{ETHGreen}{1}}$ };
        \node[state, accepting] (q11) [right = 15 mm of q10] { $\stateq'_1{\scriptstyle/\textcolor{ETHRed}{0.1}}$ };
        \draw[transition] (q10) edge[auto] node[minimum size=4mm]{ $\syma/{1}$ } (q11)
        (q11) edge[auto, loop right] node[minimum size=4mm]{ $\symb/{0.9}$ } (q11) ;
    \end{tikzpicture}
    \begin{tabular}{ll}
        \footnotesize $\pdens(\syma \symb^n)\colon$ & \footnotesize $\textcolor{ETHGreen}{1} \cdot 0.4 \cdot 0.9^n \cdot \textcolor{ETHRed}{0.1} + \textcolor{ETHGreen}{1} \cdot 0.6 \cdot 0.9^n \cdot \textcolor{ETHRed}{0.1} $
    \end{tabular}
    \caption{A determinizable \pfsaAcr{} and its deterministic version.
    }
    \label{fig:example-fslm-det}
    \end{subfigure}
    \quad
    \begin{subfigure}{0.45\textwidth}
    \begin{tikzpicture}[node distance=8mm, minimum size=12mm]
        \footnotesize
        \node[state, initial] (q0) [] { $\stateq_0{\scriptstyle/\textcolor{ETHGreen}{1}}$ };
        \node[state, accepting] (q1) [above right = of q0, xshift=13mm, yshift=-5mm] { $\stateq_1{\scriptstyle/\textcolor{ETHRed}{0.1}}$ };
        \node[state, accepting] (q2) [below right = of q0, xshift=12mm, yshift=5mm] { $\stateq_2{\scriptstyle/\textcolor{ETHRed}{0.9}}$ };
        \draw[transition] (q0) edge[auto, bend left, sloped] node[minimum size=4mm]{ $\syma/{0.4}$ } (q1)
        (q0) edge[auto, bend right, sloped] node[minimum size=4mm]{ $\syma/{0.6}$ } (q2)
        (q1) edge[auto, loop right] node[minimum size=4mm]{ $\symb/{0.9}$ } (q1)
        (q2) edge[auto, loop right] node[minimum size=4mm]{ $\symb/{0.1}$ } (q2) ;
    \end{tikzpicture}
    \begin{tabular}{ll}
        \footnotesize $\pdens(\syma \symb^n)\colon$ & \footnotesize $\textcolor{ETHGreen}{1} \cdot 0.4 \cdot 0.9^n \cdot \textcolor{ETHRed}{0.1} + \textcolor{ETHGreen}{1} \cdot 0.6 \cdot 0.1^n \cdot \textcolor{ETHRed}{0.9} $
    \end{tabular}
    \caption{A non-determinizable \pfsaAcr{}.\looseness=-1}
    \label{fig:example-fslm}
    \end{subfigure}
    \caption{Examples of \pfsaAcr{}s inducing probability distributions over $\kleene{\set{\syma, \symb}}$. Numbers in {\color{ETHGreen}green} signify initial weights, and numbers in {\color{ETHRed}red} signify final weights.}
    \label{fig:examples-fslm}
    \vspace{-15pt}
\end{figure*}

\paragraph{PFSAs as Autoregressive Language Models.}
We now show how a \pfsaAcr{} $\automaton$ induces a LM $\pLMA$ over strings $\str\in\kleene{\alphabet}$.
By \cref{def:stochastic-wfsa}, the weights of all available transitions of a \pfsaAcr{} in state $\stateq$, together with the final weight, define a probability distribution over the next action, i.e., taking a transition or halting.
We can translate this into a distribution over $\eosalphabet$ where $\eos$ corresponds to halting.
In the following, we use the notation  $\eossym\in\eosalphabet$ and $\sym\in\alphabet$ to distinguish between symbols that could be $\eos$ and those that cannot, respectively.
Specifically, we can define the probability over $\eosalphabet$ as follows
\begin{equation}
   \!\!\pLMAFun{\textcolor{ETHRed}{\eossym_\tstep}}{\textcolor{ETHPurple}{\stateq}, \str_{<t}} \!=\! \begin{cases}
        {\displaystyle \sum_{\edge{\textcolor{ETHPurple}{\stateq}}{\textcolor{ETHRed}{\eossym_\tstep}}{w}{\stateq'}}} w  & \ifcond \textcolor{ETHRed}{\eossym_\tstep}\in\alphabet\\
        \finalf\left(\textcolor{ETHPurple}{\stateq}\right) & \ifcond \textcolor{ETHRed}{\eossym_\tstep} = \eos.
    \end{cases}\label{eq:plmpfsa}
\end{equation}
Moreover, in a \pfsaAcr{}, the probability of $\eossym$ is conditionally independent of $\str_{<t}$ given the state $q$, i.e.,
\begin{equation}
\statedistributionFun{\eossym_\tstep}{\stateq, \str_{<t}} = \statedistributionFun{\eossym_\tstep}{\stateq}.\label{eq:markov}
\end{equation}
Finally, using the law of total probability, we can define an autoregressive language model as
\begin{subequations}
\begin{align}\label{eq:pfsa-autoregressive}
\!\!&\pLMAFun{\eossym_\tstep}{\strlt}\defeq \sum_{\stateq \in \states} \pLMAFun{\eossym_\tstep}{\stateq}\,\pLMAFun{\stateq}{\strlt} \nonumber \\
&= \sum_{\stateq \in \states} \pLMAFun{\eossym_\tstep}{\stateq}\,\frac{\pLMA\left(\stateq, \strlt\right)}{\pLMA\left(\strlt\right)} \\
&= \sum_{\stateq \in \states} \pLMAFun{\eossym_\tstep}{\stateq}\,\frac{\pLMA\left(\stateq, \strlt\right)}{\sum_{\stateq' \in \states} \pLMA\left(\stateq', \strlt\right)}, \label{eq:pfsa-autoregressive-last}
\end{align}
\end{subequations}
where $\pLMA\left(\stateq, \strlt\right)$ can be written as
\begin{equation}\label{eq:next-state}
    \pLMA\left(\stateq, \strlt\right) = \left(\overset{\rightarrow}{\initf}^{\top} \prod_{s=1}^t \mT^{\left(\sym_s\right)}\right)_\stateq,
\end{equation}
where $\overset{\rightarrow}{\initf}$ and $\mT^{\left(\sym_s\right)}$ refer to the vectorized initial function and symbol-specific transition matrices of the \pfsaAcr{} $\wfsa$, respectively; see \cref{app:pfsa} for a full specification.
\Cref{eq:pfsa-autoregressive} shows that a \pfsaAcr{} induces a language model $\pLMA$ as in \cref{eq:lnlm} through the conditional probabilities defined above.

\begin{definition}
  An LM $\pLM$ is a \defn{regular} language model if there exists a \pfsaAcr{} $\automaton$ whose induced language model $\pLMA$ is weakly equivalent to $\pLMA$.\looseness=-1
\end{definition}
See \cref{fig:examples-fslm} for examples of \pfsaAcr{}s that induce regular language models over $\alphabet = \set{\syma, \symb}$.

\paragraph{PFSAs and FSAs.}
Although \pfsaAcr{}s share many properties with unweighted (boolean-weighted) finite-state automata, one important difference relates to determinization.
In the unweighted case, the class of deterministic and non-deterministic FSAs are equivalent, i.e., any non-deterministic FSA has an equivalent deterministic FSA that accepts the same language.
This result, however, does not hold for \pfsaAcr{}s: There exist \pfsaAcr{}s that admit no deterministic equivalent \citep{mohri-1997-finite, Buchsbaum1998}, meaning that non-deterministic \pfsaAcr{}s are strictly more expressive than deterministic ones.
For example, the \pfsaAcr{} in \cref{fig:example-fslm} is non-deterministic and does not admit a deterministic equivalent, since $\pdens(\syma \symb^n)$ cannot be expressed as a single term for arbitrary $n \in \Nzero$.\looseness=-1

\subsection{Recurrent Neural Language Models}\label{sec:rnnlms}
Recurrent neural LMs are LMs whose conditional probabilities are given by a recurrent neural network.
Throughout this paper, we will focus on Elman RNNs \citep{Elman1990} as they are the easiest to analyze and special cases of more common networks, e.g., those based on long short-term memory \citep[LSTM;][]{10.1162/neco.1997.9.8.1735} and gated recurrent units \citep[GRUs;][]{cho-etal-2014-properties}.\looseness=-1
\begin{definition} \label{def:elman-rnn}
    An \defn{Elman RNN} (\ernnAcr{}) $\rnn = \elmanrnntuple$ is an RNN with the following hidden state recurrence:
    \begin{subequations}
    \begin{align}
    \hiddStatetzero &= \initstate  \quad\quad\quad\quad\quad\quad\quad\quad\quad\quad\,\,\,{\color{gray}(t=0)}\label{eq:elman-initialization} \\
    \hiddStatet &= \activation\left(\recMtx \vhtminus + \inMtx \inEmbedSymt + \biasVech \right) \,\, {\color{gray}(t>0)},\label{eq:elman-update-rule}
    \end{align}
    \end{subequations}
    where $\hiddStatet \in \Q^\hiddDim$ is the hidden state vector\footnote{Throughout this paper all vectors are column vectors.\looseness=-1}
    at time step $\tstep$, $\initstate \in \Q^\hiddDim $ is an initialization parameter, $\symt\in\alphabet$ is the input symbol at time step $\tstep$, $\inEmbedding\colon \alphabet \to \Q^\embedDim$ is a symbol representation function, $\recMtx \in \Q^{\hiddDim \times \hiddDim}$ and $\inMtx \in \Q^{\hiddDim \times \embedDim}$ are parameter matrices, $\biasVech \in \Q^{\hiddDim}$ is a bias vector, and $\activation\colon\Q^\hiddDim\to\Q^\hiddDim$ is an element-wise non-linear activation function.
\end{definition}
Because $\hiddStatet$ hides the string that was consumed by the Elman RNN, we also use the evocative notation $\hiddState(\str)$ to denote the result of the application of \Cref{eq:elman-update-rule} over the string $\str = \sym_1 \cdots \sym_t$.
Common examples of $\activation$ include the element-wise application of $\ReLU$, i.e., $\ReLU(x) \defeq \max(0, x)$, the Heaviside function $\heaviside(x) \defeq \ind{x > 0}$ and the sigmoid function $\sigmoid(x) \defeq \frac{1}{1 + \exp(-x)}$.
The $\ReLU$ function is the most common choice of activation in modern deep learning and the focus of our analysis in this paper \citep{Goodfellow-et-al-2016}.
To define an autoregressive language model, an Elman RNN constructs a distribution over the next symbol in $\eosalphabet$ by transforming the hidden state using some function $\mlp\colon \R^\hiddDim \to \R^\eosnsymbols$.
\begin{definition} \label{def:elman-lm}
    Let $\rnn$ be an \ernnAcr and $\mlp\colon \R^\hiddDim \to \R^\eosnsymbols$ a differentiable function.
    An \defn{Elman LM} is an LM whose conditional distributions are defined by projecting $\mlpFun{\hiddStatetminus}$ onto the probability simplex $\SimplexEosalphabetminus$ using a projection function $\projfuncEosalphabetminus\colon \R^{\eosnsymbols} \to \SimplexEosalphabetminus$:
    \begin{equation} \label{eq:next-symbol-probabilities}
        \pLMRFun{\eossym_\tstep}{\str_{<\tstep}} \defeq \projfuncEosalphabetminusFunc{\mlpFun{\hiddStatetminus}}_{\eossym_\tstep}.
    \end{equation}
    We call $\mlp$ an \defn{output function}.
\end{definition}
The vector $\mlpFun{\hiddStatetminus} \in \R^{\eosnsymbols}$ thus represents a vector of $\eosnsymbols$ values that are later projected onto $\SimplexEosalphabetminus$ to define $\pLMRFun{\eossym_\tstep}{\str_{<\tstep}}$.
The most common choice for the projection function $\projfuncEosalphabetminus$ is the \defn{softmax} defined for $\vx \in \R^\setsize, \setsize\in\N$, 
and $\idxn \in \NTo{\setsize}$ as\looseness=-1 
\begin{equation}\label{eq:softmax}
     \softmaxfunc{\vx}{\idxn} \defeq \frac{\exp \left(\invtemp \, \evx_\idxn\right)}{\sum_{\idxn' = 1}^{\setsize} \exp{\left(\invtemp \, \evx_{\idxn'}\right)}},
\end{equation}
where $\invtemp\in\R$ is called the \defn{inverse temperature} parameter.
Softmax may be viewed as the solution to the following convex optimization problem
\begin{equation}
     \softmaxfunc{\vx}{} = \argmax_{\rvz\in\Simplexnminus} \rvz^\top\vx+\entropy{\rvz},
\end{equation}
where $\entropy{\rvz}\defeq-\sum_{\idxn=1}^\setsize z_\idxn\log z_\idxn$ is the Shannon entropy.
\footnote{See \citet[Ex. 25]{boyd2004convex} for a derivation.\looseness=-1}
An important limitation of the softmax is that it implies the LM has full support, i.e., an Elman LM with a softmax projection assigns positive probability to all strings in $\kleene{\alphabet}$.
To construct Elman LMs without full support,
we also consider the \defn{sparsemax} \citep{sparsemax} as our projection function:
\begin{equation}\label{eq:spmax}
    \sparsemaxfunc{\vx}{} \defeq \argmin_{\rvz\in \Simplexnminus} \norm{\rvz - \vx}^2_2.
\end{equation}
Note that this, too, is a convex optimization problem.
In contrast to the softmax function, sparsemax is the \defn{identity} on $\Simplexnminus$, i.e., we have $\sparsemaxfunc{\vx}{} = \vx$ for $\vx \in \Simplexnminus$.

\paragraph{Common output functions.}
We now outline three special cases of $\mlp$ which will be useful in our exposition.
\begin{itemize}[itemsep=0pt]
    \item \defn{Affine output Elman LMs}. 
    A common way to implement $\mlp$ is as an affine transformation.
    In that case, $\mlp$ takes the form of $\mlpFun{\hiddState} \defeq \outMtx \hiddState + \vd$ for some $\outMtx \in \R^{\eosnsymbols \times \hiddDim}, \vd \in \R^{\eosnsymbols}$. 
    We call $\outMtx$ the \defn{output matrix}.
    \item \defn{Non-linear output Elman LMs}. 
    In this case, $\mlp$ is an arbitrary non-linear function.
    \item \defn{Deep output Elman LMs}.
    $\mlp$ is often implemented as a multi-layer perceptron in practice \citep{pascanu2014construct}.
    In that case, we call $\mlp$ a deep output Elman LM.\looseness=-1
\end{itemize}
To concretely discuss the deep output Elman LMs, we now give the definition of a multi-layer perceptron, a common non-linear model we will deploy as the output function of the Elman RNN.
Our definition is precise since we will later require an application of a universal approximation theorem. 
\begin{definition}
A \defn{multi-layer perceptron} (MLP) $\mlp\colon\R^N\to\R^M$ is a function defined as the composition of elementary functions $\vfunc_1, \cdots, \vfunc_L$ 
\begin{equation}
    \mlp\left(\vx\right) \defeq \vfunc_L \circ \vfunc_{L-1} \circ \cdots \circ \vfunc_1 \left(\vx\right),
\end{equation}
where each function $\vfunc_\ell$ for $\ell\in \NTo{L}$ is defined as
\begin{subequations}
\begin{align}
    \vfunc_\ell(\vx) &\defeq \mlpActivation\left(\mW_\ell\vx + \vb_\ell \right)\quad \ell\in \NTo{L-1} \\
    \vfunc_L(\vx) &\defeq \mW_L\vx + \vb_L,
\end{align}
\end{subequations}
where $\mW_\ell\in\R^{N_\ell\times M_\ell}$ is a weight matrix with dimensions $N_\ell$ and $M_\ell$ specific to layer $\ell$, $\vb_\ell\in\R^{M_\ell}$ is a bias vector, and $\mlpActivation$ is an element-wise non-linear activation function.
The function $\vfunc_1$ is called the \defn{input layer}, the function $\vfunc_L$ is called the \defn{output layer}, and the function $\vfunc_\ell$ for $\ell = 2, \cdots, L-1$ are called \defn{hidden layers}.\footnote{Note that we refer to MLPs by the number of hidden layers, e.g., a one-layer-MLP is an MLP with one \emph{hidden} layer.\looseness=-1
}
\end{definition}
Besides their common application in neural LMs, MLPs with various activation functions also possess well-understood approximation abilities \citep{cybenko_approximation_1989, funahashi_approximate_1989, hornik_multilayer_1989, Pinkus_1999}.
This makes their application to our goal---approximating language models---natural.
In our constructions, we concretely focus on the result by \citet{Pinkus_1999} due to its generality.
We restate it below in our own words, slightly adapting it to our use case:%
\footnote{In his original theorem, \citet{Pinkus_1999} only proves the approximation for real-valued functions $\mathbf{G}$, i.e., $M=1$. 
Our restatement trivially extends the original theorem to a range of $\R^M$ noting that one can use one large single-layer MLP consisting of $M$ separate single-layer MLPs in parallel to approximate $M$ individual functions. Furthermore, the original theorem also proves the converse direction that for polynomial $\mlpActivation$ no such MLP exists, which we drop for brevity.}
\begin{theorem}[{\citet[Theorem 3.1]{Pinkus_1999}}]\label{thm:pinkus}
    For any continuous, non-polynomial element-wise activation function $\mlpActivation$, any continuous function $\mathbf{G} :\R^N\to\R^M$, any compact subset $\sK\subset \R^N$, and any $\epsilon>0$, there exists a single-layer MLP $\mlp: \R^N \to\R^M$ with activation function $\mlpActivation$ such that $\max\limits_{\vx\in \sK}\|\mathbf{G}(\vx)-\mlp(\vx)\|_\infty<\epsilon$.\looseness=-1
\end{theorem}

\paragraph{Bounded Precision.}
We assume the hidden states $\hiddStatet$ to be rational vectors.
An important consideration in processing strings $\str \in \kleene{\alphabet}$ is then the number of bits required to represent the entries in $\hiddStatet$ and how the number of bits scales with the length of the string, $|\str|$.
This motivates the definition of precision.
\begin{definition}
The \defn{precision} $\precisionFun{\str}$ of an Elman RNN is the number of bits required to represent the entries of $\hiddState\left(\str\right)$:
\begin{equation}
    \precisionFun{\str} \defeq \max_{\idxd\in[\hiddDim]}\min_{\substack{p,q\in\N,\\ \frac{p}{q}=\hiddState\left(\str\right)_{\idxd}}} \lceil\log_2 p\rceil + \lceil\log_2 q\rceil.
\end{equation}
We say that an Elman RNN is of
\begin{itemize}[itemsep=0pt]
    \item \defn{constant precision} if $\precisionFun{\str} = \bigO{1}$, i.e., if $\precisionFun{\str} \leq C$ for all $\str \in \kleene{\alphabet}$ and some $C \in \R$,
    \item \defn{logarithmically bounded precision} if $\precisionFun{\str} = \bigO{\log|\str|}$, i.e., if there exist $\strlen_0 \in\N$ and $C\in\R$ such that for all $\str \in \kleene{\alphabet}$ with $|\str| \geq \strlen_0$, $\precisionFun{\str} \leq C\log_2{|\str|}$,
    \item \defn{linearly bounded precision} if $\precisionFun{\str} = \bigO{|\str|}$, i.e., if there exist $\strlen_0 \in\N$ and $C\in\R$ such that for all $\str \in \kleene{\alphabet}$ with $|\str| \geq \strlen_0$, $\precisionFun{\str} \leq C|\str|$, and
    \item \defn{unbounded precision} if $\precisionFun{\str}$ cannot be bounded by a function of $|\str|$.
\end{itemize}
\end{definition}
The constructions in the existing literature range from constant to unbounded precision.
The RNNs considered by \citet{svete2023recurrent}, encoding of deterministic \pfsaAcr{}s, for example, are of constant precision.
\citet{weiss-etal-2018-practical,merrill-2019-sequential,merrill-etal-2020-formal}, etc., consider models with logarithmically bounded precision.
In contrast, \citet{Siegelmann1992OnTC} and \citet{nowak2023representational} require unbounded precision to be able to represent possibly infinite running times required for the Turing completeness of RNNs.

\section{Exact Simulation} \label{sec:results}
This section presents our results on the exact simulation (that is, weak equivalence) of \fslmAcr{}s with both sparsemax as well as softmax-normalized Elman LMs. 
The results rely on \cref{thm:construction}, which provides the basis for all our subsequent results by showing that Elman RNNs with linear precision can compute the state--string distributions $\pLMA\left(\stateq, \str\right)$ (cf. \Cref{eq:next-state}) of any \pfsaAcr{} $\wfsa$.

\paragraph{A note on notation.} 
Throughout the paper, we implicitly index vectors and matrices directly with symbols for conciseness. 
Concretely, we assume an implicit ordering of the symbols and a corresponding ordering of the vector and matrix entries.
This ordering is effected through a bijection $\ordering\colon \alphabet \to \NTo{\nsymbols}$ (or any other set in place of $\alphabet$).
We then define, for $\vx \in \R^{\nsymbols}$, $\vx_\sym \defeq \vx_{\ordering\left(\sym\right)}$ and, for $\mX \in \R^{\nsymbols \times \nsymbols}$, $\mX_{\sym, \sym'} \defeq \mX_{\ordering\left(\sym\right), \ordering\left(\sym'\right)}$ for $\sym, \sym' \in \alphabet$.

\begin{restatable}{reTheorem}{construction} \label{thm:construction}
Let $\automaton$ be a \pfsaAcr{} with the state--string distribution $\pLMA\left(\stateq, \strlet\right)$ (cf. \Cref{eq:next-state}).
Then, there exists an Elman RNN $\rnn$ with linearly bounded precision such that for all $\stateq \in \states$ and all $\strlet = \sym_1 \ldots \symt \in \kleene{\alphabet}$
\begin{equation} \label{eq:construction-invariance}
    \hiddState(\strlet)_{\left(\stateq, \sym\right)} = \ind{\sym = \symt} \pLMA\left(\stateq, \strlt\right).
\end{equation}
\end{restatable}

\begin{proof}[Proof intuition.]
    The proof translates the computation of the state--string probabilities under the \pfsaAcr{} (cf. \cref{eq:next-state}) into the recurrence of the Elman RNN.
    See \cref{app:construction} for the details.
\end{proof}

\subsection{Sparsemax-normalized Elman LMs} \label{sec:ernn-sparsemax-pfsa}
We now present the first result on the representational capacity of Elman LMs of linearly bounded precision: A lower bound showing that they can implement any \fslmAcr{}, including those induced by non-deterministic \pfsaAcr{}s.
This is formally captured by the following corollary of \cref{thm:construction}.

\begin{restatable}{reCorollary}{sparsemaxExactSimulation} \label{thm:sparsemax}
Let $\automaton = \wfsatuple$ be a trim \pfsaAcr{} inducing the tight LM $\pLMA$.
Then, there exists a weakly equivalent, sparsemax-normalized non-linear output Elman LM.
\end{restatable} 
\begin{proof}[Proof intuition.]
Considering \cref{eq:pfsa-autoregressive}, $\hiddStatetminus$ contains all the information needed to compute $\pLMA\left(\eossym_\tstep \mid \strlt\right)$.
In accordance with \cref{eq:pfsa-autoregressive-last}, the hidden state is $\norm{\cdot}_1$-normalized and then multiplied by the output matrix $\outMtx$ containing the values $\pLMA\left(\eossym\mid\stateq\right)$.
See \cref{sec:exact-cor-proofs} for the full derivation.
\end{proof}

\subsection{Softmax-normalized Elman LMs}\label{sec:aernn-softmax-pfsa}

\Cref{thm:sparsemax} shows that sparsemax-normalized Elman LMs can simulate any non-deterministic \pfsaAcr{} after normalizing the hidden states to unit $\norm{\cdot}_1$ norm.
The sparsemax function, which is used in place of softmax, constitutes a crucial part of our construction. 
However, it is more standard to use the softmax in a practical implementation of an Elman RNN, but the softmax does not, unlike the sparsemax, equal the identity function on $\Simplexnminus$; $\softmax\left(\vx\right) \neq \vx$ for $\vx \in \Simplexnminus$.
Luckily, by using a different output function $\mlp$, we can achieve weak equivalence using the softmax projection function as well.
\begin{restatable}{reCorollary}{aernnExactSimulation} \label{thm:softmax}
Let $\automaton = \wfsatuple$ be a trim \pfsaAcr{} inducing the tight LM $\pLMA$.
Then, there exists a weakly equivalent, softmax-normalized non-linear output Elman LM with an $\Rex$-valued $\mlp$, where $\Rex \defeq \R \cup \set{-\infty, \infty}$ are the extended reals.
\end{restatable} 
\begin{proof}[Proof intuition.]
The idea behind this construction is similar to that of \cref{thm:sparsemax}.
Rather than the explicit $\norm{\cdot}_1$-normalization, however, this construction relies on an ($-\infty$-augmented) $\log$ transformation of the hidden state before the softmax normalization, which computes the appropriate probabilities from \cref{eq:pfsa-autoregressive}.
See \cref{sec:exact-cor-proofs} for the details.
\end{proof}

\section{Approximate Simulation} \label{sec:approximation}
\cref{sec:results} concerns itself with the exact simulation of \fslmAcr{}s with Elman LMs.
The softmax-normalized result from \cref{thm:softmax}, particularly, shows that a \aernnAcr{} with an extended-real-valued function $\mlp$ and a softmax projection function can perfectly simulate any \pfsaAcr{}.
The construction, however, relies on the use of the extended reals.
Restricting the output function $\mlp$ to the $\R$ induces approximation errors due to the full support of the softmax on real vectors.
To bring the results on the representational capacity of Elman LM even closer to practical implementations, we now discuss how an $\R$-vector-valued output function $\mlp$ can be used to \emph{approximate} \fslmAcr{}s with Elman RNNs with vanishingly small approximation error. 
We first consider $\mlp$ implemented by a custom non-linear function (the $\log$) and later as a common multi-layer perceptron.
We use total variation distance to measure the approximation error between two language models:
\begin{definition}
    Let $\pLM$ and $\qLM$ be two probability distributions over $\kleene{\alphabet}$.
    Then the \defn{total variation distance} between $\pLM$ and $\qLM$ is given by
    \begin{equation}    
        \tvd{\pLM}{\qLM}\defeq\frac{1}{2}\sum_{\str\in\kleene{\alphabet}}|\pLM(\str)-\qLM(\str)|.
    \end{equation}
\end{definition}

\subsection{Approximation with \texorpdfstring{$\log$}{log}} \label{sec:approximation-log}
We first show how \fslmAcr{}s can be approximated with $\R$-valued output functions.
The following theorem states that the $-\infty$ from \cref{thm:sparsemax} can be avoided---all outputs can be $\R$-valued---if we allow for an (arbitrarily small) approximation error.
\begin{restatable}{reTheorem}{aernnApproximateSimulation}\label{thm:aernn-approx-simulation}
    For any \pfsaAcr{} $\automaton$ with induced LM $\pLMA$ and any $\epsilon>0$, there exists a $\ReLU$-activated \aernnAcr{} $\pLM_\rnn$ with an $\R$-vector-valued function $\mlp\colon\R^\hiddDim\to\R^\hiddDim$ such that $\tvd{\pLMA}{\pLM_\rnn}< \epsilon$.
\end{restatable}
\begin{proof}[Proof intuition.]
The proof relies on approximating the arbitrary \pfsaAcr{} $\wfsa$ with a \emph{full support} \pfsaAcr{} $\wfsa'$ which approximates the LM induced by $\wfsa$ arbitrarily well in terms of the total variation distance.
Since $\wfsa'$ has full support, it can be represented exactly by a softmax-normalized Elman RNN with the same mechanism as in \cref{thm:softmax}, resulting in a recurrent LM approximating the original \pfsaAcr{} $\wfsa$.
See \cref{app:aernn-approximation-proof} for the full proof.
\end{proof}
This provides an approximation analog to \cref{thm:sparsemax}.
Due to the absence of infinities, this result is more representative of practical implementations.
Nevertheless, the use of the $\log$ output function is non-standard.
We address this in the next section.

\subsection{Approximation with \texorpdfstring{\citet{Pinkus_1999}}{Pinkus, (1999)}}\label{sec:approximation-results}
\cref{sec:approximation-log} brings the results from \cref{sec:results} to a more practical setting by considering the approximation of string probabilities.
The usage of the non-standard $\log$ transformation, however, is still unsatisfactory. 
In this section, we extend the result from the previous section to the most practically relevant setting, in which we consider the approximation of string probabilities under a \fslmAcr{} using an Elman LM with an output function in the form of an MLP.
Since such models are a common practice in modern LMs \citep{pascanu2014construct}, we consider this to be the most impactful result of our investigation.

Concretely, we show that the approximation abilities of $\ReLU$-activated deep neural networks allow us to approximate the conditional probabilities of a \pfsaAcr{} arbitrarily well, providing theoretical backing for the strong performance of softmax-normalized \ernnAcr{}s on regular languages.\footnote{Our construction applies an approximation result by \citet{Pinkus_1999}, which characterizes wide fixed-depth neural networks. An analogous result could naturally be stated for fixed-width deep $\ReLU$ networks by applying any of the approximation theorems that cover such cases \citep[e.g.,][]{10.5555/2969033.2969153,YAROTSKY2017103,pmlr-v70-raghu17a,Hanin_2019}.}\looseness=-1
\begin{restatable}{reTheorem}{dornnApproximateSimulation} \label{thm:dornn-approximation-proof}
    Let $\pLMA$ be a \fslmAcr{} induced by a \pfsaAcr{} $\automaton$ and let $\epsilon > 0$.
   Then, there exists a $\ReLU$-activated, softmax-normalized \dornnAcr{} $\pLMR$ such that 
    $\tvd{\pLMA}{\pLMR}<\epsilon$.
\end{restatable}
\begin{proof}[Proof intuition]
The final construction relies on approximating two functions: 
\begin{enumerate*}[label=\textit{(\arabic*)}]
    \item the LM induced by the arbitrary $\wfsa$ with a full support \pfsaAcr{} $\wfsa'$ and
    \item the $\log$ output function.
\end{enumerate*}
\cref{thm:aernn-approx-simulation} solves step \textit{(1)} while the application of \cref{thm:pinkus} to the approximation of $\log$ solves step \textit{(2)}.
Importantly, the full support assumption provides the necessary compact domain over which $\log$ is approximated.
See \cref{app:aernn-approximation-proof} for the full proof.
\end{proof}

\section{Discussion}
\cref{sec:results,sec:approximation} present a series of results describing the probabilistic representational capacity of Elman LMs. 
The results range from theorems that are simple, but detached from practical implementations (\cref{thm:sparsemax,thm:softmax}), to more elaborate theorems that describe the \emph{approximation} abilities of Elman LMs with modeling choices very close to those used in practice (\cref{thm:aernn-approx-simulation}).
All results lower-bound the probabilistic representational capacity of Elman LMs, i.e., they state what Elman LM \emph{can} do.
We discuss the implications of these lower bounds here.\looseness=-1

\paragraph{Approximation abilities of softmax-normalized Elman LMs.}

The most practically relevant result in our paper is the approximation result of \cref{thm:dornn-approximation-proof}.
Indeed, other results in this paper can be seen as a step-by-step build-up, each stripping away unpractical assumptions leading to this final approximation result.
\cref{thm:dornn-approximation-proof} shows that Elman LMs can approximate not only the individual conditional probabilities but rather any tight \fslmAcr{} arbitrarily well as measured by total variation distance.
Note that while arbitrarily good approximation of the \emph{conditional} probabilities does not necessarily imply the arbitrarily good approximation of full string probabilities (string probabilities are computed by multiplying the conditional probabilities of the individual symbols arbitrarily many times, resulting in a potentially large error), the assumption of tightness nevertheless allows us to make the total variation arbitrarily small---since the probability mass of sufficiently long string under a tight language model diminishes with the string length, the strings with accumulating error only contribute marginally to the total variation distance between the true \fslmAcr{} and its neural approximation.

\paragraph{Non-determinism and neural LMs.}
One of the main motivations for this investigation was capturing the non-determinism of \fslmAcr{}s with neural LMs.
Non-determinism is a prevalent notion in formal language theory as it allows both for more expressive formalisms as well as a more concise and intuitive construction of formal models of computation \citep[Ch. 2]{hopcroft2001introduction}.
However, existing work has not addressed the implications of non-determinism on the link between \fslmAcr{}s and neural LMs.
To resolve this, the constructions presented here propose direct ways of implementing non-deterministic \pfsaAcr{}s with Elman LMs by using linearly increasing the precision of hidden states, which is able to model all possible decisions of the \pfsaAcr{} directly.
This allows us to lower-bound the representational capacity of real-time Elman LMs (cf. \cref{thm:sparsemax}). This lower bound implies Elman LMs with linearly bounded precision are strictly more expressive than their counterparts with finite precision, which are only able to express \emph{deterministic} \pfsaAcr{}s \citep{svete2023recurrent}.
To the best of our knowledge, ours is the first connection that tackles non-determinism directly by emulating the computation of string probabilities in real time.\footnote{Alternatively, one could define neural LMs that simulate all accepting paths individually, which would necessitate the introduction of some form of non-determinism into the neural architecture \citep{nowak2023representational,nowak-etal-2024-computational}.}
Moreover, addressing non-determinism also allows us to \emph{compress} otherwise determinizable automata into weakly equivalent non-deterministic ones and thus implement them in possibly exponentially smaller space \citep{Buchsbaum1998} by directly simulating the smaller non-deterministic automaton with an Elman LM, a substantial improvement over existing work.\footnote{While such an LM would still require a hidden state linear in the number of states (like in the deterministic case), the number of states of the non-deterministic automaton would be smaller.\looseness=-1}

\paragraph{On the role of real-time processing.}
We limit ourselves to the real-time processing of strings by Elman LMs, as this most realistically captures how most modern LMs work \citep{weiss-etal-2018-practical}.
Moreover, real-time Elman LMs are not \emph{upper}-bounded by \fslmAcr{}s.
Using the same stack encoding function as \citet{nowak2023representational}, one can show that Elman LMs with linearly bounded precision can probabilistically generate a weighting of a context-free Dyck language of nested parentheses because simulating the required stack can be performed in real-time.\footnote{Note that all (probabilistic) context-free languages can be recognized in real-time \citep{greibachNF,abney-etal-1999-relating}.}
Moreover, by simulating multiple stacks, Elman LMs can simulate real-time Turing machines with arbitrarily many tapes \citep{nowak2023representational}, resulting in computational power exceeding that of \pfsaAcr{}s \citep{rabin1963realtime, rosenberg-1967-realtime}.
Going beyond real-time processing, a classical result by \citet{Siegelmann1992OnTC}, and extended to the probabilistic case by \citet{nowak2023representational}, establishes that RNNs with unbounded precision can simulate Turing machines.
However, the requirement of unbounded precision (and thereby computation time) is a departure from how RNNs function in practice.
\citet{nowak2023representational} place upper and lower bounds on the types of LMs $\ReLU$-activated Elman RNNs can represent.
Our real-time results stand in contrast to such Turing completeness theorems.\looseness=-1

\paragraph{Space complexity of emulating \fslmAcr{}s.}
Our constructions show that hidden states of size $\bigO{\nstates \nsymbols}$ are sufficient to simulate arbitrary \fslmAcr{}s.
While we do not address the \emph{lower bounds} on the space required for the simulation of \fslmAcr{}s, we conjecture that existing lower bounds for deterministic \pfsaAcr{}s apply here \citep[Theorems 5.1 and 5.2]{svete2023recurrent}.
Concretely, the size of the Elman LM is lower-bounded by $\min\left(\nstates, \eosnsymbols\right)$ due to the \defn{softmax bottleneck} principle \citep{yang2018breaking,svete2024theoretical,borenstein-etal-2024-what}. 
Because the \pfsaAcr{} inducing a \fslmAcr{} can, in general, define $\nstates$ probability distributions over $\eosalphabet$, a weakly equivalent Elman LM needs to model a $\Omega\left(\min\left(\nstates, \eosnsymbols\right)\right)$-dimensional space of logits to match the conditional probability distributions. 
This, in general, requires $\Omega\left(\min\left(\nstates, \eosnsymbols\right)\right)$-dimensional hidden states \citep{chang-mccallum-2022-softmax}.

\paragraph{Probabilistic Regular Languages.}
The relationship between RNNs and FSAs has attracted a lot of attention, resulting in a well-understood connection between RNNs and binary regular languages
\citep[]{Kleene1956, 10.1162/neco.1989.1.3.372, DBLP:journals/corr/abs-1906-06349, merrill-2019-sequential, merrill-etal-2020-formal} as well as between RNNs and \emph{deterministic} probabilistic regular languages \citep{svete2023recurrent}.
To the best of our knowledge, \citet{peng-etal-2018-rational} are the first to discuss recurrent neural models that can simulate the computation of string probabilities under a set of distinct \pfsaAcr{}s.
In particular, each dimension of the network’s hidden state represents the probability of the input string under a specific \pfsaAcr.
Their construction, however, does not consider the \emph{lower bound} of representational capacity of standard RNN architectures such as the Elman RNN, and, importantly, despite being weighted, is not connected to the language modeling setting.
While some aspects of our construction resemble theirs, we notably extend it to the language modeling setting, where the generalization to the probabilistic case \citep{icard-2020-calibrating} is especially relevant.\looseness=-1

\section{Conclusion}
We study a practically relevant formulation of Elman LMs with real-time processing and linear precision with respect to the string length and describe its probabilistic representational capacity in terms of \fslmAcr{}s.
Concretely, we show that Elman LMs with non-standard normalization and activation functions can simulate arbitrary non-deterministic \pfsaAcr{}s (\cref{thm:sparsemax,thm:softmax}).
We then extend these exact simulation results to \emph{approximation-based} ones, where we show that Elman LMs as used in practice can approximate \fslmAcr{}s arbitrarily well as measured by total variation distance (\cref{thm:dornn-approximation-proof}).
The intuitive constructions provide a better understanding of the representational capacity of Elman LMs and showcase the probabilistic representational capacity of Elman LMs in a new light.\looseness=-1

\section*{Limitations}
Since we only present lower bounds on the representational capacity of Elman LMs and no upper bounds, we do not provide a complete picture of their representational capacity.
We also do not consider the \emph{learnability} of \fslmAcr{}s with Elman LMs.
This crucial part of language modeling is a much harder problem that remains to be tackled.
However, many existing results from formal language theory show that learning \fslmAcr{}s based on only positive examples (as done in modern approaches to language modeling) is hard or even impossible \citep{Gold1967LanguageII, kearns-1994-cryptographic}.
The same naturally holds for the more expressive models of computation that Elman LMs can also emulate.
This puts hard bounds on what can be learned in the language modeling framework from data alone, irrespective of the concrete parameterization of the model---be it RNNs, transformers, or any other model.

The restriction to the consideration of \fslmAcr{}s of course fails to consider the non-regular phenomena of human language \citep{Chomsky57a}.
While limits of human cognition might suggest that large portions of human language can be modeled by regular mechanisms \citep{hewitt-etal-2020-rnns,svete2024theoretical}, such formalisms lack the structure and interpretability of some mechanisms higher on the Chomsky hierarchy.
However, the possibility of simulating many other real-time models of computation suggests that linearly bounded precision Elman LMs might indeed be able to capture more than just regular aspects of human language.

\section*{Ethics Statement}
The paper advances our theoretical understanding of language models.
The authors do not foresee any ethical implications of this paper.

\section*{Acknowledgements}
Ryan Cotterell acknowledges support from the Swiss National Science Foundation (SNSF) as part of the ``The Forgotten Role of Inductive Bias in Interpretability'' project.
Anej Svete is supported by the ETH AI Center Doctoral Fellowship.

\bibliography{anthology,custom}

\onecolumn
\appendix

\section{Computation of String Probabilities Under \texorpdfstring{\pfsaAcr{}s}{PFSAs}}\label{app:pfsa}

In this section, we describe how a general \pfsaAcr{} computes string probabilities using standard linear algebra notation.
This will allow for a concise connection to the Elman recurrence (cf. \cref{eq:elman-update-rule}) later.
Let $\wfsa = \wfsatuple$ be a \pfsaAcr{} and let $\str \in \kleene{\alphabet}$ be a string.
Let $\set{\mT^{\left(\sym\right)} \mid \sym \in \alphabet}$ be the set of symbol-specific transition matrices of $\wfsa$, i.e., the matrices with entries
\begin{equation}\label{eq:pfsa-trans-mtx}
\mT^{\left(\sym\right)}_{\left(\stateq, \stateq^\prime\right)} \defeq \begin{cases}
        \weightv & \ifcond \edge{\stateq}{\sym}{\weightv}{\stateq^\prime} \in \trans \\
        0        & \otherwisecondition
    \end{cases}.
\end{equation} 
That is, $\mT^{\left(\sym\right)}_{\left(\stateq, \stateq^\prime\right)}$ denotes the probability of transitioning from state $\stateq$ to state $\stateq^\prime$ with the symbol $\sym$. 
Then, the probability of a string $\str = \sym_1 \ldots \sym_\strlen \in \kleene{\alphabet}$ under $\wfsa$ is given by:
\begin{equation} \label{eq:pfsa-strprob}
    \pLMA\left(\str\right) = \overset{\rightarrow}{\initf}^\top \mT^{\left(\sym_1\right)} \mT^{\left(\sym_2\right)} \cdots \mT^{\left(\sym_{\strlen}\right)} \overset{\rightarrow}{\finalf},
\end{equation}
where $\overset{\rightarrow}{\initf} \defeq \left(\initf(\stateq)\right)_{\stateq \in \states}$ and $\overset{\rightarrow}{\finalf} \defeq \left(\finalf(\stateq)\right)_{\stateq \in \states}$ 
are vectors whose elements are the initial and final weights of the states of $\wfsa$, respectively 
\citep{peng-etal-2018-rational}.

\section{Construction of an RNN to Simulate \texorpdfstring{\pfsaAcr{}s}{PFSAs}}\label{app:construction}
\construction*
\begin{proof}
To prove \cref{thm:construction} we must, for any (possibly non-deterministic) \pfsaAcr{} $\wfsa = \wfsatuple$, construct an Elman RNN that keeps track of the state--string probabilities in its hidden states.
At a high level, we want to show that the equality \cref{eq:construction-invariance} holds for all string prefixes $\strlet \in \kleene{\alphabet}$ for all time steps $\tstep\in\N$, that is, it is an invariance that holds for all $\tstep\in\N$.
We provide the full construction of the RNN in multiple steps.
The construction uses a hidden dimension of $\hiddDim = \nstates\nsymbols$.

\paragraph{Initial condition.}
The initial hidden state, $\hiddStateZero = \initstate$ encodes the initial probability distribution over $\states$---the one defined by the initial weighting function $\initf$---but only for positions corresponding to a single arbitrary dummy input symbol $\sym_0\in\alphabet$:%
\footnote{Throughout the text, we assume that all non-specified entries of the parameters are set to $0$.}\looseness=-1
\begin{equation}\label{eq:hidstate0}
    \initstate_{(\stateq, \sym_0)} \defeq \initf(\stateq), \quad\quad\quad \forall \stateq \in \states = \pLMA\left(\stateq, \eps\right).
\end{equation}

\paragraph{Encoding the transition function.}
The Elman recurrence (cf. \cref{eq:elman-update-rule}) will effectively simulate the dynamics of the \pfsaAcr.
Intuitively, the updates to the hidden state will be analogous to parts of the computation \cref{eq:pfsa-strprob}; the individual updates to the hidden state (through the multiplication with the recurrence matrix $\recMtx$) are analogous to the multiplications with the matrices $\mT^{\left(\sym\right)}$.
However, since the RNN cannot \emph{choose} among the possible transition matrices $\set{\mT^{\left(\sym\right)} \mid \sym \in \alphabet}$ when updating the hidden state (unlike the \pfsaAcr{}), the hidden state update rule will perform \emph{all possible} transitions in parallel and then mask out the transitions that are not applicable to the current symbol.

The recurrence matrix $\recMtx \in \R^{\hiddDim \times \hiddDim}$ thus encodes the transition weights.
By having access to the hidden state $\hiddStatet$ through the matrix--vector multiplication, $\recMtx$ performs all possible transitions of the \pfsaAcr, i.e., the transitions for all symbols, in parallel.
That is, for each symbol $\sym\in\alphabet$, it computes a new distribution over states $\pLMA\left(\stateq, \strlet\sym\right)$.

The entries of the recurrence matrix are set as follows, for all $\sym^\prime\in\alphabet$:
\begin{align}
    \recMtx_{(\stateq^\prime, \sym^\prime), (\stateq, \sym)} & \defeq
    \begin{cases}
        {\weightv} & \mid \ifcond \edge{\stateq}{\sym}{\weightv}{\stateq^\prime} \in \trans \\
        0          & \mid \textbf{otherwise }
    \end{cases}.\label{eq:rec-matrix1}
\end{align}
Since $\sym^\prime$ is free, the product $\recMtx \hiddStatetminus$ copies $\statedistributionFun{\stateq}{\strlt}$ into every entry associated with $\stateq$.
Visually, the recurrence matrix $\recMtx$ consists of the transition matrices of each symbol, stacked vertically and copied horizontally, e.g., if $\alphabet=\{\syma, \symb, \cdots, \symz\}$:
\begin{align}
    \recMtx = \begin{bmatrix}
                  \recMtx^{(\syma)} & \recMtx^{(\symb)} & \cdots & \recMtx^{(\symz)}\\
                  \recMtx^{(\syma)} & \recMtx^{(\symb)} & \cdots & \recMtx^{(\symz)}\\
                  \vdots&\vdots&\vdots&\vdots\\
                  \recMtx^{(\syma)} & \recMtx^{(\symb)} & \cdots & \recMtx^{(\symz)}
              \end{bmatrix}\quad  \in \R^{\hiddDim\times\hiddDim},
\end{align}
where each transition matrix $\recMtx^{(\sym)} \defeq (\mT^{\left(\sym\right)})^\top$ has entries
\begin{align}
    \recMtx^{(\sym)}_{\stateq^\prime, \stateq} & \defeq
    \begin{cases}
        {\weightv} & \mid \ifcond \edge{\stateq}{\sym}{\weightv}{\stateq^\prime} \in \trans \\
        0          & \mid \textbf{otherwise }
    \end{cases}.\label{eq:rec-matrix2}
\end{align}
This way, multiplication with the recurrence matrix preserves the invariance in \cref{eq:construction-invariance}:
\begin{subequations}
    \begin{align}
        \left(\recMtx \, \hiddStatet\right)_{(\stateq^\prime, \sym)} &= \sum_{\edge{\stateq}{\sym}{\weightv}{\stateq^\prime} \in \trans}  \weightv \cdot \pLMA\left(\stateq, \strlet\right) \\
        &= \sum_{\stateq\in\states} \pLMA\left(\sym, \stateq^\prime \mid \stateq\right) \cdot \pLMA\left(\stateq, \strlet\right)\\
        &= \sum_{\stateq\in\states} \pLMA\left(\sym, \stateq^\prime \mid \stateq, \strlet \right) \cdot \pLMA\left(\stateq, \strlet\right)\\
        &= \sum_{\stateq\in\states} \pLMA\left(\sym, \stateq^\prime \mid \stateq, \strlet \right) \cdot \pLMA\left(\stateq \mid \strlet\right) \cdot \pLMA\left(\strlet\right)\\
        &= \left(\sum_{\stateq\in\states} \pLMA\left(\sym, \stateq^\prime \mid \stateq, \strlet \right) \cdot \pLMA\left(\stateq \mid \strlet\right) \right) \cdot \pLMA\left(\strlet\right)\\
        &= \pLMA\left(\sym, \stateq^\prime \mid \strlet \right) \cdot \pLMA\left(\strlet\right)\\
        &= \pLMA\left(\sym, \stateq^\prime, \strlet\right) \\
        &= \pLMA\left(\stateq', \strlet \sym\right).
    \end{align}
\end{subequations}
We now have a distribution over states and symbols given the history.
However, at time step $\tstep$ we also observe the actual input symbol, meaning that we want to only retain the distribution over states for the observed symbol.

We use the input matrix $\inMtx$ and the bias vector $\biasVech$ to zero out parts of the hidden state so that it only retains the probabilities of states reachable by reading the current symbol $\symt$ and sets the rest to $0$.
Since the $\ReLU$ activation function zeroes out negative values, it will be sufficient to subtract an appropriate bias term from the entries associated with unobserved symbols.
This can be achieved by defining, for any $\stateq \in \states$ and $\sym \in \alphabet$:
\begin{subequations}
    \begin{align}
        \inMtx_{(\stateq, \sym), (\sym)} & \defeq 1,                              \label{eq:inmtx-sym}                      \\
        \biasVech                                                & \defeq -\one_{\hiddDim}, \label{eq:bias}
    \end{align}
\end{subequations}
where $\one_{\hiddDim}$ is a $\hiddDim$-dimensional vector of ones.
This ensures that the negative bias values are balanced by the product $\inMtx \inEmbedSymt$ exactly for the tuple entries associated with symbol $\symt$:
\begin{subequations}
    \begin{align}
        \hiddState\left(\strlet \sym\right)_{\left(\stateq', \sym'\right)} &= \ReLU\left(\recMtx_{\left(\stateq', \sym'\right)} \hiddState\left(\strlet\right) + \inMtx_{\left(\stateq', \sym'\right)} \onehot{\sym} + \bias_{\left(\stateq', \sym'\right)} \right) \\
        &= \ReLU\left(\pLMA\left(\stateq', \strlet \sym\right) + \ind{\sym' = \sym} - 1\right) \\
        &= \ind{\sym' = \sym}\pLMA\left(\stateq', \strlet \sym\right).
    \end{align}
\end{subequations}
Since probability scores cannot be greater than 1, the values corresponding to symbols that were not observed will be zeroed out by the activation function.

\paragraph{Linearly Bounded Precision.}
Finally, we want to show that the RNN requires linearly bounded precision, i.e., there exist $\strlen_0\in\N$ and $C \in\R_{\geq 0}$ such that for all $\str \in \kleene{\alphabet}$ with $|\str| \geq\strlen_0$, $\precisionFun{\str} \leq  C|\str|$.
In the following, we write the precision required by a rational $x\in\Q$ as 
\begin{equation}
    \precision(x)\defeq\min\limits_{\substack{p,q\in\N,\\ \frac{p}{q}=x}} \lceil\log_2 p\rceil + \lceil\log_2 q\rceil.
\end{equation}
Note that $\precision(xy)\leq\precision(x)+\precision(y)$ for all $x,y\in\Q$.

Two remarks will be important for the proof:
\begin{enumerate}
    \item The number of bits required to represent any dimension of $\inMtx \inEmbedSymt + \biasVech$ is upper bounded by a constant $C^\prime$ independent of $\tstep$, given by
    \begin{equation}
        C^\prime \defeq \max_{\sym\in\alphabet}\max_{\idxd\in[\hiddDim]}\precision([\inMtx\inEmbedSym+\biasVech]_\idxd).
    \end{equation}
    \item The $\ReLU$ either leaves a dimension of the hidden state unchanged or sets it to 0, meaning it never increases the required precision of its argument.
\end{enumerate}

We show linear precision by induction over $|\str|$ for $\str \in \kleene{\alphabet}$ with $|\str| \geq \strlen_0 \defeq 1$, choosing $C$ as
\begin{equation}
        C \defeq \precision(\hiddDim) + \max_{\idxi,\idxj\in[\hiddDim]}\precision(\recMtx_{\idxi\idxj}) + \max_{\idxd}\precision(\initstate_\idxd) + C^\prime.
\end{equation}

\begin{itemize}
    \item \textbf{Base case:} For $\str = \sym$ for some $\sym \in \alphabet$, 
    \begin{subequations}
    \begin{align}
    \precisionFun{\sym}&=\max_{\idxd\in[\hiddDim]}\precision(\hiddState_{1,\idxd}) \\
        &= \max_{\idxd\in[\hiddDim]}\precision(\ReLU(\recMtx\initstate + \inMtx \inEmbeddingFun{\sym} + \biasVech)_\idxd)\\
        &\leq \max_{\idxd\in[\hiddDim]}\precision((\recMtx\initstate)_\idxd) + C^\prime\\
        &\leq \max_{\idxd,\idxi,\idxj,\in[\hiddDim]}\precision(\hiddDim \recMtx_{i,j} \initstate_\idxd) + C^\prime\\
        &\leq \precision(\hiddDim) + \max_{\idxi,\idxj\in[\hiddDim]}\precision(\recMtx_{\idxi\idxj}) + \max_{\idxd\in[\hiddDim]}\precision(\initstate_\idxd) + C^\prime \\
        &= C.
    \end{align}
    \end{subequations}
    \item \textbf{Induction step:} Assume that $\precisionFun{\str} \leq C |\str|$ for all $\str \in \kleene{\alphabet}$ with $|\str| = \strlen$.
    Let $\str \in \kleene{\alphabet}$ with $|\str| = \strlen$ be such a string.
    Then, for $\str' \defeq \str \sym'$ with $\sym' \in \alphabet$ (a string of length $\strlen + 1$), we have
    \begin{subequations}
    \begin{align}
        \precisionFun{\str'} &= \max_{\idxd\in[\hiddDim]}\precision(\hiddState_{\tstep{+}1,\idxd}) \\
        &= \max_{\idxd\in[\hiddDim]}\precision(\ReLU\left(\recMtx \vht + \inMtx \inEmbeddingFun{\sym'} + \biasVech \right)_\idxd)\\
        &\leq \max_{\idxd\in[\hiddDim]}\precision((\recMtx\vht)_\idxd) + C^\prime\\
        &\leq \precision(\hiddDim) + \max_{\idxi,\idxj\in[\hiddDim]}\precision(\recMtx_{\idxi\idxj}) + \underbrace{\max_{\idxd\in[\hiddDim]}\precision(\hiddStatetminus_\idxd)}_{=\precisionFun{\str}} + C^\prime \\
        &\leq C + C|\str|\\
        &=C(\str + 1) \\
        &=C\str'.
    \end{align}
    \end{subequations}
    Therefore, any $\ReLU$-activated Elman RNN over the rationals can be represented with linear precision. 
    Note that for $\tstep=0$, $\hiddStatetzero = \initstate$, which requires constant precision depending only on the choice of $\initstate$ which is included in $C$, so above we actually showed a stronger bound of $\precisionFun{\tstep}\leq C(\tstep+1)$ that applies to all time steps.
    Observe also that we use the fact that each computational step in our model corresponds to reading a symbol in $\alphabet$; this is in contrast to models such as those investigated in \citet{nowak2023representational}, which are allowed to consume the empty string $\eps$.
\end{itemize}

\end{proof}

\section{Proofs of Exact Simulation} \label{sec:exact-cor-proofs}

\sparsemaxExactSimulation*
\begin{proof}
Let $\pLMA\left(\stateq, \str\right)$ be $\automaton$'s state--string distribution (cf. \Cref{eq:next-state}).
By \Cref{thm:construction}, there exists an Elman RNN $\rnn = \elmanrnntuple$ with the hidden states\looseness=-1
\begin{equation}
    \hiddState(\strlt)_{(\stateq, \sym)} = \ind{\sym = \symtminus} \pLMA\left(\stateq, \strlt\right)
\end{equation}
for all $\strlt \in \kleene{\alphabet}$.

We now define a simple non-linear transformation $\mlp$ that transforms $\hiddState(\strlt)$ into values that are later normalized to $\pLMA\left(\sym_\tstep\mid \strlt\right)$ by the sparsemax.
The function $\mlp$ is a composition of two functions: \begin{enumerate*}[label=\textit{(\arabic*)}]
    \item $\norm{\cdot}_1$-normalization of the hidden state and
    \item linear transformation of the normalized hidden state that sums the outgoing transitions from $\stateq \in \states$ that emit $\eossym_\tstep$.
\end{enumerate*}
$\mlp$ will this take the form
\begin{equation}
    \mlp\left(\hiddState\right) \defeq \outMtx \; \frac{\hiddState}{\norm{\hiddState}_1}
\end{equation}
for some $\outMtx \in \R^{\eosnsymbols \times \hiddDim}$ that we define below.

Defining $\hiddState \defeq \hiddStatet = \hiddState\left(\strlt\right)$ and applying $\norm{\cdot}_1$ normalization to $\hiddState$ as $\hiddState' \defeq \frac{\hiddState}{\norm{\hiddState}_1}$ results in
\begin{subequations}
\begin{align}
    \hiddState'_{(\stateq, \sym)} &= \frac{\ind{\sym = \symtminus} \pLMA\left(\stateq, \strlt\right)}{\sum_{\idxd^\prime}\abs{\hiddState_{\idxd^\prime}}}\\
    &= \frac{\ind{\sym = \symtminus} \pLMA\left(\stateq, \strlt\right)}{\sum_{\stateq\in\states, \sym' \in \alphabet}\abs{\hiddState_{(\stateq^\prime, \sym')}}}\\
    &= \frac{\ind{\sym = \symtminus} \pLMA\left(\stateq, \strlt\right)}{\sum_{\stateq^\prime\in\states, \sym' \in \alphabet} \ind{\sym' = \symtminus} \pLMA\left(\stateq', \strlt\right)}\\
    &= \frac{\ind{\sym = \symtminus} \pLMA\left(\stateq, \strlt\right)}{\sum_{\stateq^\prime\in\states} \pLMA\left(\stateq', \strlt\right)}\\
    &= \ind{\sym = \symtminus} \statedistributionFun{\stateq}{\strlt},
\end{align}
\end{subequations}
i.e., the $\hiddState'$ holds a distribution over states given all the input symbols up to time step $\tstep - 1$. We define $\outMtx$ for $\stateq\in\states, \sym\in\alphabet, \eossym\in\eosalphabet$ as
\begin{align}\label{eq:output-matrix}
    \outMtx_{\eossym, (\stateq, \sym)} \defeq \begin{cases}
        \sum_{\edge{\stateq}{\sym}{\weightv}{\stateq'} \in\trans} \weightv & \ifcond \eossym\in\alphabet \\
        \finalfFun{\stateq}  & \ifcond \eossym = \eos 
    \end{cases} = 
    \begin{cases}
        \pLMA\left(\sym\mid \stateq\right) & \ifcond \eossym\in\alphabet \\
        \pLMA\left(\eos\mid \stateq\right)  & \ifcond \eossym = \eos 
    \end{cases}.
\end{align}
Computing the matrix--vector product $\outMtx \, \hiddState'\left(\strlt\right)$ reveals that $\left(\outMtx \, \hiddState'\left(\strlt\right)\right)_{\eossym} = \pLMA\left(\eossym\mid\strlt\right)$:
\begin{subequations}
    \begin{align}
        \left(\outMtx \, \hiddState'\left(\strlt\right)\right)_{\eossym} 
        &= \sum_{\idxd = 1}^{\hiddDim} \outMtx_{\eossym, \idxd} \hiddState'\left(\strlt\right)_{\idxd} \\
        &= \sum_{\stateq \in \states, \sym \in \alphabet} \outMtx_{\eossym, \left(\stateq, \sym\right)} \hiddState'\left(\strlt\right)_{\left(\stateq, \sym\right)} \\
        &= \sum_{\stateq \in \states, \sym \in \alphabet} \pLMA\left(\eossym\mid \stateq\right) \ind{\sym = \symtminus} \statedistributionFun{\stateq}{\strlt} \\
        &= \sum_{\stateq \in \states} \pLMA\left(\eossym\mid \stateq\right) \statedistributionFun{\stateq}{\strlt} \\
        &= \sum_{\stateq \in \states} \pLMA\left(\eossym\mid \stateq, \strlt\right) \statedistributionFun{\stateq}{\strlt} \\
        &= \pLMA\left(\eossym\mid\strlt\right).
    \end{align}
\end{subequations}
The identity of the sparsemax on $\SimplexEosalphabetminus$ then gives us
\begin{subequations}
    \begin{align}
        \pLMR(\eossym \mid \strlet) &\defeq \sparsemaxfunc{\outMtx \, \hiddState(\strlet)}{\eossym} \\
        &= \pLMA(\eossym \mid \strlet)
    \end{align}
\end{subequations}
for all $\strlt \in \kleene{\alphabet}, \eossym \in \eosalphabet$.
This implies by \cref{eq:lnlm} that $\pLMR$ and $\pLMA$ are weakly equivalent.
\end{proof}

\aernnExactSimulation*
\begin{proof}
Like in the proof of \cref{thm:sparsemax}, we construct the function $\mlp$ that transforms the hidden state $\hiddStatetminus$ into $\eosnsymbols$ values that are later normalized into $\pLMA\left(\symt\mid\hiddStatet\right)$, this time with the softmax.
Concretely, we want to compute the autoregressive next-symbol probabilities $\pLMA\left(\symt\mid \strlt\right)$.
Assuming that all values are positive (we relax this later), we can write:
\begin{subequations}
\begin{align}
    \pLMA\left(\eossym_\tstep\mid \strlt\right) 
    &= \frac{\pLMA\left(\strlt \eossym_\tstep\right)}{\pLMA\left(\strlt\right)} \\
    &= \frac{\sum_{\stateq \in \states} \pLMA\left(\stateq, \strlt \eossym_\tstep\right)}{\sum_{\stateq \in \states} \pLMA\left(\stateq, \strlt\right)} \\
    &= \frac{\sum_{\stateq \in \states} \pLMA\left(\stateq, \strlt \right) \pLMA\left(\eossym_\tstep \mid \stateq, \strlt\right)}{\sum_{\stateq \in \states} \pLMA\left(\stateq, \strlt\right)} \\
    &= \frac{\sum_{\stateq \in \states} \pLMA\left(\stateq, \strlt \right) \pLMA\left(\eossym_\tstep \mid \stateq\right)}{\sum_{\stateq \in \states} \pLMA\left(\stateq, \strlt\right)} \\
    &= \exp\left(\log\left(\sum_{\stateq \in \states} \pLMA\left(\stateq, \strlt \right) \pLMA\left(\eossym_\tstep \mid \stateq\right)\right) - \log\left(\sum_{\stateq \in \states} \pLMA\left(\stateq, \strlt\right)\right)\right). \label{eq:softmax-last}
\end{align}
\end{subequations}
Intuitively, $\pLMA\left(\symt\mid \strlt\right)$ can thus be computed by normalizing the values $\log\left(\sum_{\stateq \in \states} \pLMA\left(\stateq, \strlt \right) \pLMA\left(\eossym_\tstep \mid \stateq\right)\right) - \log\left(\sum_{\stateq \in \states} \pLMA\left(\stateq, \strlt\right)\right)$ for $\eossym_\tstep \in \eosalphabet$ with the softmax.
Defining the vector $\vv_\eossym \defeq \log\left(\sum_{\stateq \in \states} \pLMA\left(\stateq, \strlt \right) \pLMA\left(\eossym_\tstep \mid \stateq\right)\right) - \log\left(\sum_{\stateq \in \states} \pLMA\left(\stateq, \strlt\right)\right)$ for $\eossym_\tstep \in \eosalphabet$, we therefore have\looseness=-1
\begin{equation}
    \pLMA\left(\symt\mid \strlt\right) = \softmaxfunc{\vv}{\eossym}.
\end{equation}
However, due to the invariance of the softmax to the addition of vectors of the form $c \cdot \one$, for any $c \in \R$, we also have that 
\begin{equation}
    \pLMA\left(\symt\mid \strlt\right) = \softmaxfunc{\vv + \log\left(\sum_{\stateq \in \states} \pLMA\left(\stateq, \strlt\right)\right) \cdot \one}{\eossym}.
\end{equation}
Entries of $\vv' \defeq \vv + \log\left(\sum_{\stateq \in \states} \pLMA\left(\stateq, \strlt\right)\right) \cdot \one$ equal
\begin{equation}
    \vv'_\eossym \defeq \vv_\eossym + \log\left(\sum_{\stateq \in \states} \pLMA\left(\stateq, \strlt\right)\right) = \log\left(\sum_{\stateq \in \states} \pLMA\left(\stateq, \strlt \right) \pLMA\left(\eossym_\tstep \mid \stateq\right)\right).
\end{equation}
The goal, then, is to compute the values $\sum_{\stateq \in \states} \pLMA\left(\stateq, \strlt \right) \pLMA\left(\eossym_\tstep \mid \stateq\right)$, logarithmically transform them, and then normalize them with the softmax function.
To achieve this, we take the same output matrix $\outMtx \in \R^{\eosnsymbols \times \hiddDim}$ as in the proof of \cref{thm:sparsemax} (cf. \cref{eq:output-matrix}).
The matrix--vector product $\outMtx \hiddState$ contains the values 
\begin{subequations}
\begin{align}
    \left(\outMtx \hiddState\right)_\sym  
        &= \sum_{\idxd = 1}^{\hiddDim} \outMtx'_{\eossym, \idxd} \hiddState\left(\strlt\right)_{\idxd} \\
        &= \sum_{\stateq \in \states, \sym \in \alphabet} \outMtx'_{\eossym, \left(\stateq, \sym\right)} \hiddState\left(\strlt\right)_{\left(\stateq, \sym\right)} \\
        &= \sum_{\stateq \in \states, \sym \in \alphabet} \pLMA\left(\eossym\mid \stateq\right) \ind{\sym = \symtminus} \pLM\left(\stateq, \strlt\right) \\
        &= \sum_{\stateq \in \states} \pLMA\left(\eossym\mid \stateq\right) \pLM\left(\stateq, \strlt\right).
\end{align}
\end{subequations}
These are the values we are interested in (from \cref{eq:softmax-last}).

To allow for zero-valued arguments to the $\log$, we now define the non-linear transformation $\overline{\log}$ as
\begin{equation} \label{eq:aernn-logprobs}
    \overline{\log}\left(\vx\right)_d \defeq \begin{cases}
        \log\left(\vx \right)_d  & \ifcond \vx_d  > 0\\
        -\infty & \otherwisecondition
    \end{cases}.
\end{equation}
Note that the use of $-\infty$ here is what necessitates us to require the extended reals, in contrast to in \Cref{thm:sparsemax}.  
This motivates the definition of $\mlp\colon \R^\hiddState \to \Rex^\eosnsymbols$ as
\begin{equation} \label{eq:log-exact-mlp}
    \mlp\left(\hiddState\right) \defeq \overline{\log} \left(\outMtx' \hiddState\right).
\end{equation}
The application of the softmax to $\mlp\left(\hiddState\right)$ results in 
\begin{subequations}
    \begin{align}
        \pLMR\left(\eossym \mid \strlt\right) 
        &= \softmaxfunc{\mlp\left(\hiddState\right)}{\eossym} \\
        &= \frac{\exp\left(\overline{\log}\left(\pLMA\left(\strlt\eossym\right)\right)\right)}{\sum_{\eossym' \in \eosalphabet} \exp\left(\overline{\log}\left(\pLMA\left(\strlt\eossym'\right)\right)\right)} \\
        &= \frac{\pLMA\left(\strlt\eossym\right)}{\sum_{\eossym' \in \eosalphabet} \pLMA\left(\strlt\eossym'\right)} \\
        &= \pLMA\left(\eossym \mid \strlt\right),
    \end{align}
\end{subequations}
implying that $\pLMA$ and $\pLMR$ are weakly equivalent, as desired.
\end{proof}

\section{Proof of Approximate Simulation by Elman LMs with Softmax over the Real Numbers}\label{app:aernn-approximation-proof}

Recall that we assume, without loss of generality, that for each pair of states $\stateq,\stateq^\prime$ and each symbol $\sym\in\alphabet$ there exists a transition $\edge{\stateq}{\sym}{w}{\stateq^\prime}$, albeit potentially with zero weight, $w=0$.
For any \pfsaAcr{}, we define the following corresponding fully connected\footnote{Fully connected here means the \pfsaAcr{} has nonzero transition probability for all symbols between all pairs of strings, and nonzero final weights for all states.} automaton:
\begin{definition}
    Given a \pfsaAcr{} $\automaton=\wfsatuple$ and $\delta\in \R$, 
    $\automaton$'s \defn{\deltapfsaAcr} $\automatondelta=\left(\alphabet, \states, \trans_\delta, \initf, \finalf_\delta\right)$ is a fully connected \pfsaAcr{} derived from $\automaton$ by 
    \begin{enumerate}[label=(\arabic*.)]
        \item adding $\delta$ to all transition weights and final weights and
        \item locally re-normalizing the outgoing transition weights and the final weight of each state such that they sum to $1$.
    \end{enumerate}
\end{definition}

See \cref{alg:deltapfsa} for the pseudocode of transforming a \pfsaAcr{} to a \deltapfsaAcr.
\begin{algorithm}[h]
\begin{algorithmic}[1]
\Func{\deltapfsaAcr($\automaton=\wfsatuple$, $\delta$)}
\State $\trans_\delta\leftarrow\{\}$
\State $\finalf_\delta(\stateq)\leftarrow 0\quad \forall\stateq\in\states$
\For{$\edge{\stateq}{\sym}{w}{\stateq^\prime}\in\automaton$}
    \State $w_\delta \leftarrow \frac{w+\delta}{1+(\nsymbols\nstates+1)\delta}$\label{todeltapfsa:newweight}
    \State Add transition $\edge{\stateq}{\sym}{w_\delta}{\stateq^\prime}$ to $\trans_\delta$.\label{todeltapfsa:deltatrans}
\EndFor
\For{$\stateq\in\states$}
\State $\finalf_\delta(\stateq)\leftarrow\frac{\finalfFun{\stateq}+\delta}{1+(\nsymbols\nstates+1)\delta}$\label{todeltapfsa:deltarho}
\EndFor
\State $\automatondelta\leftarrow\left( \alphabet, \states, \trans_\delta, \initf, \finalf_\delta \right)$
\State \Return $\automatondelta$
\EndFunc
\end{algorithmic}
\caption{Conversion of a \pfsaAcr{} to a \deltapfsaAcr}
\label{alg:deltapfsa}
\end{algorithm}

\begin{definition}
Given two LMs $\pLM$ and $\qLM$ over $\kleene{\alphabet}$ and a subset $S\subset\kleene{\alphabet}$, the \defn{restricted \tvdacr} on $S$ is defined as\looseness=-1
\begin{equation}
    \rtvd{S}{\pLM}{\qLM}\defeq\frac{1}{2}\sum_{\str\in S}|\pLM(\str) - \qLM(\str)|.
\end{equation}
\end{definition}

\begin{lemma}\label{lem:continuity}
    Let $\automaton=\wfsatuple$ be a \pfsaAcr{} inducing the LM $\pLMA$ and $\automatondelta$ $\automaton$'s \deltapfsaAcr inducing the LM $\pLMdelta$ for some $\delta \in \R$.\footnote{In this paper, we are only interested in small $\delta > 0$. 
    However, the weighted automaton $\automatondelta$ is well defined for any $\delta \in \R$, albeit with negative weights in some cases. 
    In the case of negative weights, $\automatondelta$ is no longer a \emph{probabilistic} FSA.
    }
    For any finite subset of strings $S\subset\kleene{\alphabet}$, we have that the restricted total variation distance $\rtvd{S}{\pLMA}{\pLMdelta}$ is continuous at $\delta = 0$.
\end{lemma}
\begin{proof}
Let $\paths(\str)$ be the set of all paths in $\automaton$ that accept $\str$.
Because $\automaton$ and $\automatondelta$ have the same topology, we will also use $\paths(\str)$ for the set of all strings in $\automatondelta$ that accept $\str$.
Note that $\prevq(\apath)$ is a finite set since both $\automaton$ and $\automatondelta$ is a real-time \pfsaAcr{}.
Now, to show that $\tvd{\pLMA}{\pLMdelta}$ is continuous at $\delta=0$, we need to show $\lim_{\delta \rightarrow 0} \rtvd{S}{\pLMA}{\pLMdelta}$ exists.
This follows directly from the following manipulations:
\begin{subequations}
\begin{align}
\lim_{\delta \rightarrow 0} \rtvd{S}{\pLMA}{\pLMdelta}&= \lim_{\delta \rightarrow 0}  \frac{1}{2}\sum_{\str\in S}|\pLMA(\str) - \pLMdelta(\str)|\\
    &= \lim_{\delta \rightarrow 0} \frac{1}{2}\sum_{\str\in S}\left|\sum_{\apath\in\paths(\str)}\prod_{\idx = 0}^{|\str|+1} w_n-\sum_{\apath\in\paths(\str)}\prod_{\idx = 0}^{|\str|+1} w^{\delta}_n\right| \\
    &= \lim_{\delta \rightarrow 0} \frac{1}{2}\sum_{\str\in S}\left|\sum_{\apath\in\paths(\str)}\prod_{\idx = 0}^{|\str|+1} w_n-\sum_{\apath\in\paths(\str)}\prod_{\idx = 0}^{|\str|+1} \frac{w_n + \delta}{1 + \delta(|\alphabet||\states| + 1)}\right| \label{eq:limit-outside} \\
    &=  \frac{1}{2}\sum_{\str\in S}\left|\sum_{\apath\in\paths(\str)}\prod_{\idx = 0}^{|\str|+1} w_n-\sum_{\apath\in\paths(\str)}\prod_{\idx = 0}^{|\str|+1} \frac{w_n + \lim_{\delta \rightarrow 0} \delta}{1 + \lim_{\delta \rightarrow 0} \delta(|\alphabet||\states| + 1)}\right| \label{eq:limit-inside} \\
    &=  \frac{1}{2}\sum_{\str\in S}\left|\sum_{\apath\in\paths(\str)}\prod_{\idx = 0}^{|\str|+1} w_n-\sum_{\apath\in\paths(\str)}\prod_{\idx = 0}^{|\str|+1} \frac{w_n + 0}{1 + 0}\right| \\
    &=  \frac{1}{2}\sum_{\str\in S}\left|\sum_{\apath\in\paths(\str)}\prod_{\idx = 0}^{|\str|+1} w_n-\sum_{\apath\in\paths(\str)}\prod_{\idx = 0}^{|\str|+1} w_n\right|  \\
    &= 0,
\end{align}
\end{subequations}
where the step from \Cref{eq:limit-inside} to \Cref{eq:limit-outside} follows from the basic limit laws (continuous functions preserve limits), as desired.
\end{proof}

\begin{lemma}\label{lem:lmpartition}
    Let $\automaton$ be a trim \pfsaAcr{} inducing the LM $\pLMA$.
    For any $\epsilon>0$, we can partition the set of all strings $\kleene{\alphabet}$ into a finite set $S$ and its complement $\kleene{\alphabet}\setminus S$ such that $\pLMA\left(\kleene{\alphabet}\setminus S\right) = \sum\limits_{\str\in\kleene{\alphabet}\setminus S}\pLMA(\str)<\epsilon$.
\end{lemma}
\begin{proof}
By \citet[Thm 5.3]{du-etal-2023-measure}, a \pfsaAcr{} induces
a tight language model if and only if every accessible state is also co-accessible.
Thus, because we assume $\automaton$ is trim, the above-stated condition holds and $\pLMA$ is tight.
Moreover, because $\pLMA$ is tight, we have
\begin{equation}
1 = \sum_{\str \in \kleene{\alphabet}} \pLMA(\str)  = \sum_{n=0}^\infty \sum_{\str \in  \alphabet^{n}} \pLMA(\str)
\end{equation} 
Since the infinite sum converges, we have that for any $\epsilon > 0$, there exists an $M < \infty$ such that
\begin{equation}
\sum_{n=M + 1}^\infty \sum_{\str \in  \alphabet^{n}} \pLMA(\str) < \epsilon
\end{equation}
and thus
\begin{equation}
\left|1 - \sum_{n=0}^M \sum_{\str \in  \alphabet^{n}} \pLMA(\str)\right| < \epsilon.
\end{equation}
Defining $S \defeq \bigcup_{n=0}^M \alphabet^n$ completes the proof since $\sum_{n=0}^M \sum_{\str \in  \alphabet^{n}} \pLMA(\str) = \sum_{\str \in  S} \pLMA(\str) \geq 1-\epsilon$ and $\sum_{\str \in \kleene{\alphabet} \setminus S} \pLMA(\str) < \epsilon$.
\end{proof}

\begin{theorem}\label{thm:almost-there}
    For any \pfsaAcr{} $\automaton$ inducing the LM $\pLMA$ and any $\epsilon>0$, there exists a rational $\eta>0$ such that $\tvd{\pLMA}{\pLMeta}<\epsilon$, where $\pLMeta$ is the LM of $\automatoneta$, the $\eta$-perturbed $\automaton$.
\end{theorem}
\begin{proof}
Let $\epsilon>0$.
Choose a \emph{finite} set $S \subset \kleene{\alphabet}$
such that $p(S) > 1 - \frac{\epsilon}{2}$ and $p(\kleene{\alphabet} \setminus S) < \frac{\epsilon}{2}$.
By \Cref{lem:lmpartition}, we can always find such a set $S$.
Next, by \Cref{lem:continuity}, $\rtvd{S}{\pLMA}{\pLMdelta} \colon \R \rightarrow \R$ is continuous at $\delta=0$.
Thus, there exists a $\delta'$ such that $|\eta - 0| < \delta'$ 
implies $\rtvd{S}{\pLMA}{\pLMeta} < \frac{\epsilon}{4}$.
Choose a rational $\eta$ positive such that $\eta < \delta'$.
Next, we partition $\tvd{\pLMA}{\pLMeta}$ into three additive terms and bound each term individually
\begin{subequations}
\begin{align}
    \tvd{\pLMA}{\pLMeta} &= \frac{1}{2}\sum_{\str\in \kleene{\alphabet}}|\pLMA(\str) - \pLMeta(\str)| \\
    &= \frac{1}{2}\sum_{\str\in S}|\pLMA(\str) - \pLMeta(\str)|  + \frac{1}{2}\sum_{\str\in \kleene{\alphabet} \setminus S}|\pLMA(\str) - \pLMeta(\str)| \label{subeq:tvdbound10}\\
    &\leq \frac{1}{2}\sum_{\str\in S}|\pLMA(\str) - \pLMeta(\str)|  + \frac{1}{2}\sum_{\str\in \kleene{\alphabet} \setminus S}|\pLMA(\str)| + \frac{1}{2}\sum_{\str\in \kleene{\alphabet} \setminus S}|\pLMeta(\str)| \label{subeq:tvdbound20}\\
    &= \underbrace{\frac{1}{2}\sum_{\str\in S}|\pLMA(\str) - \pLMeta(\str)|}_{< \epsilon/4}  + \frac{1}{2}\underbrace{\sum_{\str\in \kleene{\alphabet} \setminus S}\pLMA(\str)}_{<\epsilon/2} + \frac{1}{2}\underbrace{\sum_{\str\in \kleene{\alphabet} \setminus S}\pLMeta(\str)}_{<\epsilon} \label{subeq:tvdbound30}\\
  &< \frac{1}{4}\epsilon  + \frac{1}{4} \epsilon + \frac{1}{2}\epsilon \label{subeq:tvdbound50}\\ 
  &=  \epsilon,
\end{align}
\end{subequations}
where, in the step from \Cref{subeq:tvdbound30} to \Cref{subeq:tvdbound50}, the first two terms are bounded by the choice of $S$ while the bound on the third term follows from the claim below.
This proves the result.
\end{proof}
\begin{claim}\label{claim:perturbed-tail}
We claim that $\epsilon > \pLMeta(\kleene{\alphabet}\setminus S)$.
\end{claim}
\begin{proof}
The proof follows from simple manipulations:
\begin{subequations}
\begin{align}
\frac{\epsilon}{4}>\rtvd{S}{\pLMA}{\pLMeta} &=\frac{1}{2}\sum_{\str\in S}\left|\pLMA(\str) - \pLMeta(\str)\right| \geq \frac{1}{2}\left|\sum_{\str\in S}\left(\pLMA(\str) - \pLMeta(\str)\right)\right| \\
&=\frac{1}{2}\left|\pLMA(S)-\pLMeta(S)\right| =\frac{1}{2}\left|1 - \pLMA(\kleene{\alphabet}\setminus S)-1 - \pLMeta(\kleene{\alphabet}\setminus S)\right|\\ &=\frac{1}{2}\left| \pLMA(\kleene{\alphabet}\setminus S) - \pLMeta(\kleene{\alphabet}\setminus S)\right|.
\end{align}
\end{subequations}
This results in
 $\frac{\epsilon}{2} > \left| \pLMA(\kleene{\alphabet}\setminus S) - \pLMeta(\kleene{\alphabet}\setminus S)\right|$.
Additionally, we have $\frac{\epsilon}{2} > \left| \pLMA(\kleene{\alphabet}\setminus S) \right|$.
Summing these together, we arrive at 
\begin{subequations}
\begin{align}
    \epsilon &> \left| \pLMA(\kleene{\alphabet}\setminus S) - \pLMeta(\kleene{\alphabet}\setminus S)\right| + \left| \pLMA(\kleene{\alphabet}\setminus S) \right| \\
    &\geq \left| \pLMA(\kleene{\alphabet}\setminus S) - \pLMeta(\kleene{\alphabet}\setminus S) + \pLMA(\kleene{\alphabet}\setminus S) \right| = \left| \pLMeta(\kleene{\alphabet}\setminus S)\right|
\end{align}
\end{subequations}
Thus, we arrive at $\epsilon > \pLMeta(\kleene{\alphabet}\setminus S)$.
\end{proof}


Next, we define the notion of Lipschitz continuity.
\begin{definition}
    Let $\sK\subseteq\R^N$ be a subset of $\R^N$.
    A mapping $f\colon \sK\to\R^M$ is Lipschitz (or L-Lipschitz) with respect to a norm $\|\cdot\|_p$ if there exists a constant $L>0$ such that for all $\vx, \vy \in \sK$,
    \begin{equation}
        \|f(\vx)-f(\vy)\|_2\leq L\|\vx-\vy\|_2.
    \end{equation}
\end{definition}

\begin{proposition}[\citet{gao2018properties}]
\label{prop:softmax-lipschitz}
The softmax function $\softmaxshort$ is $\invtemp$-Lipschitz with respect to $\|\cdot\|_2$ i.e., for all $\vx, \vx^\prime \in \R^N$,
\begin{equation}
    \|\softmaxshort(\vx)-\softmaxshort(\vx^\prime)\|_2 \leq \invtemp\|\vx - \vx^\prime\|_2,
\end{equation}
where $\invtemp$ is the inverse temperature parameter as defined in \cref{eq:softmax}.
\end{proposition}

Softmax-normalized Elman LMs (cf. \cref{def:elman-lm}) define the conditional next-symbol probabilities $\pLM\left(\symt\mid\strlt\right)$ by normalizing $\mlp\left(\hiddState\left(\strlt\right)\right)$.
To abstract away the particular hidden states and their transformations, we now introduce the following definition.
\begin{definition} \label{def:f-lm}
    Let $\str \in \kleene{\alphabet}$ be a string of length $T$.
    Let $\mX\in\R^{\hiddDim \times (\strlen{+}1)}$ be a matrix of $\strlen{+}1$ column vectors of size $\hiddDim$:\looseness=-1
    \begin{equation}
        \mX \defeq \begin{bmatrix}
        \vert & \vert &  & \vert \\
        \vx(1)   & \vx(2) & \cdots &  \vx(\strlen{+}1)\\
        \vert & \vert &  & \vert 
        \end{bmatrix}.
    \end{equation}
    Define the function $\fstr\colon \R^{\hiddDim \times (\strlen{+}1)}  \rightarrow [0,1]$ that induces an autoregressive LM $\pLM$:
    \begin{subequations}
    \begin{align}
    \fstr(\mX) &= \pLM(\eos \mid \str) \prod_{\tstep=1}^\strlen \pLM(\eossym_\tstep \mid \strlt) \\
    &= \softmaxshort(\vx(\strlen{+}1))_{\eos} \prod_{\tstep=1}^\strlen \softmaxshort(\vx(\tstep))_{\symt} \\
    &= \prod_{\tstep=1}^{\strlen+1} \softmaxshort(\vx(\tstep))_{\eossym_\tstep} \label{eq:product-of-softmaxes}  \\
    &\defeq \prod_{\tstep=1}^{\strlen+1} \frac{\exp \vx(\tstep)_{\eossym_\tstep}}{\sum_{\symoverline \in \eosalphabet} \exp \vx(\tstep)_{\symoverline}},
    \end{align}
    \end{subequations}
    where $\eossym_{\strlen+1} \defeq \eos$.
\end{definition}

\begin{lemma}\label{lemma:liptschitz}
The function $\fstr$ from \cref{def:f-lm} is Lipschitz continuous with respect to $||\cdot||_{\infty}$ with Lipschitz constant $L_{\str}=\bigO{|\str|}$, which, as the notation suggests, depends on the string $\str$.
\end{lemma}

\begin{proof}
    The function $\fstr(\mX)$ is a product of softmax functions. 
    Therefore, we divide the proof into two Steps.
    In Step 1, we start by showing that a single softmax $\softmaxshort$ is Lipschitz continuous with respect to $||\cdot||_{\infty}$.
    In Step 2, we show that products of softmax functions, e.g., $\fstr(\mX)$, are Lipschitz continuous by induction.\looseness=-1
    \paragraph{Step 1: Lipschitz continuity of $\softmaxshort$.}
    
\Cref{prop:softmax-lipschitz} states that $\softmaxshort \colon \R^N \rightarrow \R^N$ is Lipschitz with respect to the norm $||\cdot||_2$, with Lipschitz constant $\invtemp$.
This implies the following inequality
\begin{equation}
    |\softmaxshort(\vx)_n - \softmaxshort(\vy)_n| \leq \invtemp ||\vx - \vy||_2 \quad \forall n\in \N.
\end{equation}
We generalize this result to $||\cdot||_{\infty}$ (and, indeed, all norms on $\R^N$) by noting the equivalence of norms on finite-dimensional vector spaces, i.e., there exists a constant $C$ such that
$||\vx||_2 \leq C ||\vx||_{\infty}$ for all $\vx \in \R^N$.
Thus, we get
\begin{equation}\label{eq:pointwise-norm}
  |\softmaxshort(\vx)_n - \softmaxshort(\vy)_n| \leq \invtemp ||\vx - \vy||_2 \leq \underbrace{C \invtemp}_{\defeq L_{\softmaxshort}} ||\vx - \vy||_{\infty}.
\end{equation}
Taking a max over the left-hand side of \Cref{eq:pointwise-norm}, we get
\begin{equation}
    ||\softmaxshort(\vx) - \softmaxshort(\vy)||_{\infty} \leq L_{\softmaxshort}||\vx - \vy||_{\infty}.
\end{equation}
 \paragraph{Step 2: Lipschitz continuity of $\fstr(\mX)$.}
    The function $\fstr(\mX)$, as defined in \Cref{eq:product-of-softmaxes} is a product of $\strlen + 1$ softmax functions.
    By definition, each component of the softmax $\softmaxshort_d$ for $d\in[\hiddDim]$ has a bounded range $[0,1]$.
    We now prove by induction that the product of $T$ softmax functions is Lipschitz continuous.
    Define $f_N\colon\R^{\hiddDim\times\strlen{+}1}\to\R$ as $f_N(\mX)\defeq \prod_{t=1}^{N}\|\softmaxshort(\vx(t))\|_\infty$, where $N\leq \strlen{+}1$ and $\vx(t)$ is the $t^{\text{th}}$ column vector of $\mX$. 
    Clearly, $f_\strlen(\mX)$ upper bounds $\fstr(\mX)$.
    \begin{itemize}
    \item  \textbf{Base case:} For $N=1$, $f_1(\mX)=\|\softmaxshort(\vx(1))\|_\infty$. Therefore, 
    \begin{equation}
        |f_1(\mX) - f_1(\mY)| = \|\softmaxshort(\vx(1)-\vy(1))\|_\infty \leq L_\softmaxshort \|\vx(1)-\vy(1)\|,
    \end{equation} where $L_\softmaxshort$ is the Lipschitz constant as shown in Step 1. 
    Note that $\vx(1)$ and $\vy(1)$ are submatrices of $\mX$ and $\mY$, respectively, so $\|\vx(1)-\vy(1)\|_\infty \leq \|\mX-\mY\|_\infty$. Therefore,
    \begin{equation}
        |f_1(\mX) - f_1(\mY)| \leq L_\softmaxshort \|\mX-\mY\|_\infty.
    \end{equation} 
    \item \textbf{Induction step:}
    Assume that for $N$, $f_N(\mX) = \prod_{\tstep=1}^{N} \|\softmaxshort(\vx(t))\|_\infty$ is Lipschitz with constant $N L_{\softmaxshort}$. 
    We show that $f_N(\mX) = \prod_{t=1}^{N+1} \|\softmaxshort(\vx(\tstep))\|_\infty$ is Lipschitz with constant $(N{+}1) L_{\softmaxshort}$. 
    
    Consider $\mX, \mY \in \R^{\hiddDim\times\strlen{+}1}$. By the definition of $f$, 
    \begin{subequations}
    \begin{align}
    |&f_{N{+}1}(\mX)-f_{N{+}1}(\mY)| \\
    &=|f_N(\mX)\|\softmaxshort(\vx(N{+}1))\|_\infty - f_N(\mY) \|\softmaxshort(\vy(N{+}1))\|_\infty| \\ 
    &=|f_N(\mX)(\|\softmaxshort(\vx(N{+}1))\|_\infty - \|\softmaxshort(\vy(N{+}1))\|_\infty) + (f_N(\mX) - f_N(\mY))\|\softmaxshort(\vx(N{+}1))\|_\infty| \\ 
    &\leq|f_N(\mX)||\|\softmaxshort(\vx(N{+}1))\|_\infty - \|\softmaxshort(\vy(N{+}1))\|_\infty| + |f_N(\mX) - f_N(\mY)|\|\softmaxshort(\vx(N{+}1))\|_\infty\\
    &\leq|\|\softmaxshort(\vx(N{+}1))\|_\infty - \|\softmaxshort(\vy(N{+}1))\|_\infty| + |f_N(\mX) - f_N(\mY)|\label{eq:bound-one}\\
    &\leq|\|\softmaxshort(\vx(N{+}1)) - \softmaxshort(\vy(N{+}1))\|_\infty| + |f_N(\mX) - f_N(\mY)| \label{eq:inverse-triangle}\\
    &\leq L_\softmaxshort\|\vx(N{+}1) - \vy(N{+}1))\|_\infty + |f_N(\mX) - f_N(\mY)|\label{eq:step-one}\\
    &\leq L_\softmaxshort\|\vx(N{+}1) - \vy(N{+}1))\|_\infty + NL_\softmaxshort\|\mX-\mY\|_\infty \label{eq:IH}\\
    &\leq L_\softmaxshort\|\mX-\mY\|_\infty + NL_\softmaxshort\|\mX-\mY\|_\infty \label{eq:bound-matrix}\\
    &= (N+1)L_\softmaxshort\|\mX-\mY\|_\infty,  
    \end{align}
    \end{subequations}
    where in \cref{eq:bound-one} we use the fact that $\|\softmaxshort(\cdot)\|_\infty \leq 1$, and hence also $f_N(\mX)\leq 1$ for all $N \in \N$.
    In \cref{eq:inverse-triangle}, we use the inverse triangle inequality, in \cref{eq:step-one} we use the result from Step 1, and in \cref{eq:IH}, we use the induction hypothesis. 
    Finally, in \cref{eq:bound-matrix}, we use the fact that the infinity norm of a matrix bounds the infinity norm of a subset of that matrix.
    \end{itemize}
Thus, for $\str \in \kleene{\alphabet}$, we have $\fstr$ is Lipschitz continuous with Lipschitz constant $L_\str = |\strlen{+}1| L_\softmaxshort = |\strlen{+}1| C\invtemp = \bigO{|\str|}$.\looseness=-1
\end{proof}

\aernnApproximateSimulation*
\begin{proof}
    Let $\automaton$ be a \pfsaAcr{} with induced LM $\pLMA$, and be the LM induced by the  $\automaton$ and $\epsilon > 0$. 
    By \cref{thm:almost-there}, we can construct an $\eta$-perturbed \pfsaAcr{} $\automatoneta$ inducing the full-support LM $\pLMeta$ such that $\tvd{\pLMA}{\pLMeta}<\epsilon$.
    Using the same construction as that in the proof of \cref{thm:softmax}, except that we replace $\overline{\log}$ in \cref{eq:aernn-logprobs} with $\log$, we can construct a non-linear output Elman LM $\pLMR$ (the Elman RNN identical to the one in \cref{thm:construction}) that is \emph{weakly equivalent} to the full-support $\pLMeta$.
    This is because, since $\pLMeta$ has full support, all outputs of $\mlp$ are positive, and for any $x>0, \log(x)=\overline{\log}(x)$.
    We then have that 
    \begin{align}
        \tvd{\pLMA}{\pLMR}
        \leq \underbrace{\tvd{\pLMA}{\pLMeta}}_{< \eps} + \underbrace{\tvd{\pLMeta}{\pLMR}}_{=0} < \eps,
    \end{align}
    which finishes the proof.
\end{proof}

\section{Proof of Approximate Simulation by Softmax-normalized, Deep Output Elman LMs}\label{app:dornn-approximation-proof}

\dornnApproximateSimulation*
\begin{proof}
Fix $\epsilon > 0$ and let $\pLMdelta$ be the language model induced from a $\delta$-perturbation of a \pfsaAcr{} $\automaton$ weakly equivalent to $\pLMA$; we will specify $\delta$ at a later point in the proof.
Our goal is to show that there exists a $\ReLU$-activated, softmax-normalized \dornnAcr{} $\pLMR$ such that $\tvd{\pLMA}{\pLMR} = \frac{1}{2}\sum_{\str\in \kleene{\alphabet}}|\pLMA(\str) - \pLMR(\str)|$.
The RNN $\rnn$ will be the same as the one in \cref{thm:construction}; we only seek to approximate the \emph{output} function of the Elman LM.
To do so, we first partition $\tvd{\pLMA}{\pLMR}$ into four additive terms and bound each term individually:
\begin{subequations}
\begin{align}
    &\tvd{\pLMA}{\pLMR} = \frac{1}{2}\sum_{\str\in \kleene{\alphabet}}|\pLMA(\str) - \pLMR(\str)| \\
&= \frac{1}{2}\sum_{\str\in \kleene{\alphabet}}|\pLMA(\str) - \pLMdelta(\str) + \pLMdelta(\str) - \pLMR(\str)| \\
&\leq \frac{1}{2}\sum_{\str\in \kleene{\alphabet}}|\pLMA(\str) - \pLMdelta(\str)| + \frac{1}{2}\sum_{\str\in \kleene{\alphabet}}|\pLMdelta(\str) - \pLMR(\str)| \\
&= \frac{1}{2}\sum_{\str\in \kleene{\alphabet}}|\pLMA(\str) - \pLMdelta(\str)| + \frac{1}{2}\sum_{\str\in S}|\pLMdelta(\str) - \pLMR(\str)| + \frac{1}{2}\sum_{\str\in \kleene{\alphabet} \setminus S}|\pLMdelta(\str) - \pLMR(\str)| \\
&\leq \underbrace{\frac{1}{2}\sum_{\str\in \kleene{\alphabet}}|\pLMA(\str) - \pLMdelta(\str)|}_{< \frac{\epsilon}{4} \text{ by Step 1}} + \underbrace{\frac{1}{2}\sum_{\str\in S}|\pLMdelta(\str) - \pLMR(\str)|}_{< \frac{\epsilon}{4} \text{ by Step 2}} + \underbrace{\frac{1}{2}\sum_{\str\in \kleene{\alphabet} \setminus S}|\pLMdelta(\str)|}_{< \frac{\epsilon}{4} \text{ by Step 3}} + \underbrace{\frac{1}{2}\sum_{\str\in \kleene{\alphabet} \setminus S}|\pLMR(\str)|}_{< \frac{\epsilon}{4} \text{ by Step 4}}\nonumber  \\
&< \epsilon.
\end{align}
\end{subequations}
The bounds are shown in the corresponding steps below. \\

\noindent \emph{Caution}: Steps 1, 2, 3, and 4 all depend on $\delta$, and Steps 2, 3, and 4 all depend on $S$.
To decouple the individual steps (for expository purposes), in each step, we will choose a separate $\delta_i$ and then assign $\delta = \min(\delta_1, \delta_2, \delta_3, \delta_4)$ and a separate set $S_i$
and then assign $S = S_2 \cup S_3 \cup S_4$.

\paragraph{Step 1.}
By \Cref{thm:almost-there}, we can choose a $\delta_1$ such that $\frac{1}{2}\sum_{\str\in \kleene{\alphabet}}|\pLMA(\str) - \pLMdeltaN{1}(\str)| < \frac{\epsilon}{4}$.

\paragraph{Step 2.}
Let $S_2$ be \emph{any} finite subset of $\kleene{\alphabet}$ and choose $\delta_2 \in (0, 1)$ arbitrarily. 
We first note that the $\delta_2$-perturbed language model $\pLMdeltaN{2}$ can be written as follows
\begin{subequations}
\begin{align}
\pLMdeltaN{2}(\str) &= \pLMdeltaN{2}(\eos \mid \str) \prod_{t=1}^T \pLMdeltaN{2}(y_t \mid \str_{<t})  \\
&= \softmaxshort(\log(\outMtx\hiddState(\str)))_{\eos} \prod_{t=1}^T \softmaxshort(\log(\outMtx \hiddState(\str_{<t})))_{y_t}.
\end{align}
\end{subequations}
Moreover, since $\pLMdeltaN{2}$ is induced from an $\delta_2$-perturbed \pfsaAcr{}, $\hiddState(\str)_q > 0$ for all $q \in Q$ and all $\str \in \kleene{\alphabet}$.
Next, define the following constants
\begin{subequations}
\begin{align}
\xi_1 &\defeq \min_{\str \in S_2} \min_{q \in Q} \hiddState(\str)_{q} > 0 \label{eq:xione} \\
\xi_2 &\defeq \min_{\str \in S_2} \min_{\eossym \in \eosalphabet} \log (\outMtx \hiddState(\str))_{\eossym} < 0
\end{align}
\end{subequations} 
Finally, we will define
\begin{equation}
    L = \max_{\str \in S_2} L_{\str}
\end{equation}
where $L_{\str}$ is the Lipschitz constant of the function $\fstr$ (cf. \cref{def:f-lm}) given in \Cref{lemma:liptschitz}.
Our goal is to find a $\ReLU$ MLP $\mlp$ that is an arbitrarily good approximator of $\hiddState \mapsto \log(\outMtx \hiddState) \colon [\xi_1, 1]^{|\states|} \rightarrow [\xi_2, 0]^{\eosnsymbols}$ where the $\log$ is applied element-wise.
Because $[\xi_1, 1]^{|\states|} \ni \hiddState \mapsto \log(\outMtx \hiddState)$ is continuous and its domain is compact, \cref{thm:pinkus} tells us there does exist  a one-layer MLP $\mlp_1$ such that the following bound holds
\begin{equation}\label{eq:pinkus-bound}
    \max_{\hiddState \in [\xi_1, 1]^{|\states|}} ||\log(\outMtx\hiddState) - \mlp_1(\hiddState)||_{\infty} < \frac{\epsilon}{2 L |S|}
\end{equation}
In order to make sure that the one-layer MLP $\mlp$ only receives inputs in its domain $[\xi_1, 1]^{|\states|}$, we compose $\mlp_1$ with another $\ReLU$ layer $\mlp_2$ defined element-wise as follows
\begin{equation}
  \mlp_2(\hiddState)_i \defeq \ReLU(0, \hiddState_i - \xi_1) + \xi_1
\end{equation}
which computes the function $\max(\xi_1, x)$.
Because of the definition of $\xi_i$ in \Cref{eq:xione}, 
the composition of $\mlp_1$ with $\mlp_2$ has no effect on arguments taken from the set $\{ \hiddState(\str) \mid \str \in S_2\}$.
Now, define $\mlp(\hiddState) \defeq \mlp_1(\mlp_2(\hiddState))$.
Then, making use of the finiteness of $S_2$, for any $\str \in S_2$, we have that $\fstr$ is $L_{\str}$-Lipschitz continuous and a fortiori $L$-Lipschitz continuous by \Cref{lem:continuity}.
Combining the result with the \Cref{eq:pinkus-bound}, we have that
\begin{align}
    |\fstr(\log \outMtx\hiddState(\str)) - \fstr(F(\hiddState(\str)))| &< \frac{\epsilon}{2 L |S_2|} L = \frac{\epsilon}{2|S_2|} 
\end{align}
which, with no more than an adjustment of notation (see \Cref{def:f-lm}), tells us that
\begin{equation}
|\pLMdeltaN{2}(\str) - \pLMR(\str)| < \frac{\epsilon}{2 L |S_2|} L = \frac{\epsilon}{2|S_2|}
\end{equation}
Because we chose a finite set $S_2$ of strings, we arrive at the following bound
\begin{equation}
    \frac{1}{2} \sum_{\str \in S_2} |\pLMdeltaN{2}(\str) - \pLMR(\str)| <    \frac{1}{2}\sum_{\str \in S_2} \frac{\epsilon}{|S_2|} = \frac{1}{2} |S_2| \frac{\epsilon}{2|S_2|} = \frac{\epsilon}{4}
\end{equation}
as desired.

\paragraph{Step 3.}
Again by \Cref{thm:almost-there}, we can choose $\delta_3$ such that $\frac{1}{2}\sum_{\str\in \kleene{\alphabet}}|\pLMA(\str) - \pLMdeltaN{3}(\str)| < \frac{\epsilon}{8}$.
The proof of \Cref{thm:almost-there} yields a set $S_3$ such that $\pLMA(S_3) = 1 - \frac{\epsilon}{16}$.
Then, we can apply \Cref{claim:perturbed-tail}, which says $\frac{\epsilon}{2} > \pLMdeltaN{3}(\kleene{\alphabet} \setminus S_3)$.

\paragraph{Step 4.}
By Step 2, for any finite set $S_4\subset \kleene{\alphabet}$, we can choose a $\delta_4>0$ such that $\frac{\epsilon}{8} > \frac{1}{2}\sum_{\str\in S_4}\left|\pLMdeltaN{4}(\str) - \pLMR(\str)\right|$.
Then, consider the following chain of inequalities 
\begin{subequations}
\begin{align}
\frac{\epsilon}{8} &>\frac{1}{2}\sum_{\str\in S_4}\left|\pLMdeltaN{4}(\str) - \pLMR(\str)\right| \\
&\geq \frac{1}{2}\left|\sum_{\str\in S_4}\left(\pLMdeltaN{4}(\str) - \pLMR(\str)\right)\right| \\
&=\frac{1}{2}\left|\pLMdeltaN{4}(S_4)-\pLMR(S_4))\right| \\
&=\frac{1}{2}\left|1 - \pLMdeltaN{4}(\kleene{\alphabet}\setminus S_4)- (1 - \pLMR(\kleene{\alphabet}\setminus S_4))\right|\\ &=\frac{1}{2}\left| \pLMdeltaN{4}(\kleene{\alphabet}\setminus S_4) - \pLMR(\kleene{\alphabet}\setminus S_4)\right|.
\end{align}
\end{subequations}
This results in
 $\frac{\epsilon}{4} > \left| \pLMdeltaN{4}(\kleene{\alphabet}\setminus S_4) - \pLMR(\kleene{\alphabet}\setminus S_4)\right|$.\\
So far, we only required $S_4$ to be finite but did not fix its value.
Now, by \cref{claim:perturbed-tail}, we can choose $S_4$ such that we have $\frac{\epsilon}{4} > \left| \pLMdeltaN{4}(\kleene{\alphabet}\setminus S_4) \right|$.
Summing both of these together, we arrive at 
\begin{subequations}
\begin{align}
    \frac{\epsilon}{2} &> \left| \pLMdeltaN{4}(\kleene{\alphabet}\setminus S_4) - \pLMR(\kleene{\alphabet}\setminus S_4)\right| + \left| \pLMdeltaN{4}(\kleene{\alphabet}\setminus S_4) \right| \\
    &\geq \left| \pLMdeltaN{4}(\kleene{\alphabet}\setminus S_4) - \pLMR(\kleene{\alphabet}\setminus S_4) + \pLMdeltaN{4}(\kleene{\alphabet}\setminus S_4) \right| = \left| \pLMR(\kleene{\alphabet}\setminus S_4)\right|.
\end{align}
\end{subequations}
Thus, we have $\frac{\epsilon}{2} > \pLMR(\kleene{\alphabet}\setminus S_4)$ because probabilities are non-negative.
\end{proof}

\end{document}